\documentclass[runningheads,envcountsame, a4paper]{llncs}
\usepackage[pdftex]{graphicx}
\usepackage{epstopdf}
\usepackage{amssymb}
\usepackage{amsmath}
\usepackage{booktabs}
\usepackage{subfigure}
\usepackage[hyphens]{url}
\usepackage[misc]{ifsym}

\newcommand{\argmax}{\operatornamewithlimits{\arg \max}}
\newcommand{\argmin}{\operatornamewithlimits{\arg \min}}

\newcommand{\local}{{\textrm{local}}}

\newcommand{\bbE}{{\mathbb{E}}}
\newcommand{\bbP}{{\mathbb{P}}}
\newcommand{\bbR}{{\mathbb{R}}}

\newcommand{\bs}{{\mathbf{s}}}
\newcommand{\bx}{{\mathbf{x}}}
\newcommand{\by}{{\mathbf{y}}}
\newcommand{\bX}{{\mathbf{X}}}

\newcommand{\bI}{{\boldsymbol I}}
\newcommand{\bk}{{\boldsymbol k}}
\newcommand{\bK}{{\boldsymbol K}}
\newcommand{\bL}{{\boldsymbol L}}
\newcommand{\bxi}{{\boldsymbol \xi}}

\newcommand{\calD}{{\mathcal{D}}}
\newcommand{\calN}{{\mathcal{N}}}
\newcommand{\calX}{{\mathcal{X}}}

\newcommand{\figref}[1]{Fig.~\ref{#1}}
\newcommand{\tabref}[1]{Table~\ref{#1}}
\newcommand{\corref}[1]{Corollary~\ref{#1}}
\newcommand{\thmref}[1]{Theorem~\ref{#1}}
\newcommand{\lemref}[1]{Lemma~\ref{#1}}
\newcommand{\defref}[1]{Definition~\ref{#1}}

\spnewtheorem{defn}{Definition}{\bfseries}{\itshape}
\spnewtheorem{thm}{Theorem}{\bfseries}{\itshape}
\spnewtheorem{lem}{Lemma}{\bfseries}{\itshape}
\spnewtheorem{cor}{Corollary}{\bfseries}{\itshape}

\graphicspath{{./figures/}}

\begin{document}
\title{On Local Optimizers of Acquisition Functions\\in Bayesian Optimization}
\titlerunning{On Local Optimizers of Acquisition Functions in Bayesian Optimization}
\author{Jungtaek Kim\inst{1}\textsuperscript{(\Letter)} \and
Seungjin Choi\inst{2}}
\authorrunning{J. Kim and S. Choi}
\institute{Pohang University of Science and Technology, Pohang, Republic of Korea\\
\email{jtkim@postech.ac.kr}
\and
Inference Lab, BARO AI, Seoul, Republic of Korea\\
\email{seungjin@baroai.com}}
\maketitle
\setcounter{footnote}{0}
%
%
%
% /* abstract */
\begin{abstract}
Bayesian optimization is a sample-efficient method for finding a global optimum 
of an expensive-to-evaluate black-box function.
A global solution is found by accumulating a pair of query point and its function value,
repeating these two procedures:
(i) modeling a surrogate function;
(ii) maximizing an acquisition function to determine where next to query.
Convergence guarantees are only valid when the global optimizer 
of the acquisition function is found at each round
and selected as the next query point.
In practice, however, local optimizers of an acquisition function are also used, 
since searching for the global optimizer is often a non-trivial or time-consuming task.
In this paper we consider three popular acquisition functions, PI, EI, and GP-UCB 
induced by Gaussian process regression. 
Then we present a performance analysis on the behavior of local optimizers 
of those acquisition functions, in terms of {\em instantaneous regrets} over global optimizers.
We also introduce an analysis, 
allowing a local optimization method to start from multiple different initial conditions.
Numerical experiments confirm the validity of our theoretical analysis.

\end{abstract}
%
%
%
% /* main manuscript begins */
\section{Introduction\label{sec:intro}}

Bayesian optimization provides an efficient method for finding a global optimum of an objective function
$f(\bx): \calX \rightarrow \bbR$, defined over a compact set $\calX \subset \bbR^d $:
\begin{equation}
\bx^\dagger = \argmin_{\bx \in \calX} f(\bx),
\label{eqn:naive_go}
\end{equation}
where, in general, $f(\bx)$ is a black-box function, i.e., its closed-form expression is not available
and its gradient is not available either.
The value of the function can be computed at a query point $\bx$ but the evaluation requires a high cost.
In this paper we assume that the objective function $f(\bx)$ of interest is Lipschitz-continuous.
Bayesian optimization searches for a minimum of $f(\bx)$ to solve the problem \eqref{eqn:naive_go}, 
gradually accumulating $(\bx_t, f(\bx_t))$ where input points $\bx_t$ are carefully chosen and 
corresponding function values $f(\bx_t)$ are calculated at $\bx_t$.
It provides an efficient approach in terms of the number of function evaluations required.

In Bayesian optimization, a global solution to the problem \eqref{eqn:naive_go} is determined by repeating the following two procedures.
At each round, we first train a probabilistic model\footnote{Gaussian process regression is used in this paper.}
using the data observed so far to construct a surrogate function for $f(\bx)$.
Then we define an acquisition function~\cite{KushnerHJ1964jbe,MockusJ1978tgo,SrinivasN2010icml} over the domain $\calX$, 
which accounts for the utility provided 
by possible outcomes drawn from the distribution determined by the surrogate model.
The maximization of an acquisition function, referred to as an {\em inner optimization,} yields the selection of 
the next query point at which to evaluate the objective function. 
Convergence guarantees are only valid when the global optimizer of the acquisition function is found
and selected as the next query point.
In practice, however, local optimizers of acquisition functions are also used, since searching for the exact
optimizer of the acquisition function is often a non-trivial or time-consuming task.

A recent work~\cite{WilsonJT2018neurips} has addressed the acquisition function optimization, 
elucidating gradient-based optimization of Monte Carlo estimates of acquisition functions,
as well as on sub-modularity for a family of maximal myopic acquisition functions.
However, so far, there is no study on what the performance loss is when a local optimizer of an
acquisition function is selected as the next query point.
In this paper we attempt to provide an answer to this question on the performance loss brought by
local optimizers of acquisition functions over global optimizers, in terms of instantaneous regrets.
To this end, we consider three different solutions to the maximization of an acquisition function:
(i) a global optimizer; (ii) a local optimizer; (iii) a multi-started local optimizer. 
For performance analysis of local optimizers, with respect to the global optimizer, we define 
an {\em instantaneous regret difference} for a local optimizer as well as for a multi-started local optimizer
and present its bound for each case.
As expected, the multi-started local optimizer yields a tighter bound on the instantaneous regret difference,
compared to the one for a local optimizer.

In this paper we consider three popular acquisition functions, probability of improvement (PI)~\cite{KushnerHJ1964jbe}, 
expected improvement (EI)~\cite{MockusJ1978tgo}, 
and Gaussian process upper confidence bound (GP-UCB)~\cite{SrinivasN2010icml}, each of which is calculated
by posterior mean and variance determined by Gaussian process regression.
The main contribution of this paper is summarized as:
\begin{itemize}
\item 
We provide an upper bound on the instantaneous regret difference between global and local optimizers, which is given in \thmref{thm:first};
\item 
We provide an upper bound on the instantaneous regret difference when a multi-started local optimization method 
is employed to search for a local maximum of the acquisition function, which is given in \thmref{thm:second};
\item 
Numerical experiments are provided to justify our theoretical analyses.
\end{itemize}

\section{Background\label{sec:background}}

In this section, we briefly review Bayesian optimization, the detailed
overview of which is referred to \cite{BrochuE2010arxiv,ShahriariB2016procieee,FrazierPI2018arxiv},
and define instantaneous regret difference that
is used as a performance measure for local optimizers of acquisition functions.
We also explain global and local optimization methods that are popularly used 
to search for maxima of acquisition functions.

\subsection{Bayesian Optimization\label{subsec:bo}}

The Bayesian optimization strategy solves the problem \eqref{eqn:naive_go}, by gradually selecting
queries $\bx_1, \ldots, \bx_T$ and their corresponding noisy evaluations $y_1, \ldots, y_T$
where $y_t = f(\bx_t) + \epsilon_t$ 
with $\epsilon_t \sim \calN(0,\sigma_n^2)$, such that a minimizer of
$f(\bx)$ is determined from $\{\bx_1, \ldots, \bx_T\}$.
Given the data  $\calD_{t-1} = \{ (\bx_1, y_1), \ldots, (\bx_{t-1},y_{t-1}) \}$ observed up to round $t-1$, 
the next point $\bx_t$ is chosen as a maximizer of acquisition function $a(\bx | \calD_{t-1})$, i.e.,
\begin{equation}
\bx_t = \argmax \, a(\bx | \calD_{t-1}).
\label{eqn:aqmax}
\end{equation}

The acquisition function is the expected utility $u$ of a query $\bx$:
\begin{equation}
a(\bx | \calD_{t-1}) = \int u(\bx, y) \, p(y | \bx, \calD_{t-1}) \, \mathrm{d}y,
\end{equation}
where the posterior distribution $p(y | \bx, \calD_{t-1})$ is calculated by Gaussian process regression using $\calD_{t-1}$ here.

Solving \eqref{eqn:aqmax} is another optimization problem appearing in the Bayesian optimization task 
given in \eqref{eqn:naive_go}.
We consider three different solutions to the maximization of an acquisition function,
defined in detail below.

% Definition 1
\begin{defn}[Global optimizer]
	\label{def:go_af}
	We denote by $\bx_{t,g}$ the optimizer of the acquisition function $a(\bx | \calD_{t-1})$ at round $t$, determined
	by a global optimization method, given a time budget $\tau$:
	\begin{equation}
		\bx_{t, g} = \stackrel{\textrm{global}}{\argmax}_{\bx \in \calX} a(\bx | \calD_{t-1}).
		\label{eqn:global_acq}
	\end{equation}
	$\bx_{t, g}$ is referred to as a global optimizer.
\end{defn}

% Definition 2
\begin{defn}[Local optimizer]
	\label{def:lo_af}
	We denote by $\bx_{t, l}$ the optimizer of the acquisition function $a(\bx | \calD_{t-1})$ at round $t$, 
	determined by an iterative (local) optimization method where the convergence meets 
	${\| \bx_{t, l}^{(\tau)} - \bx_{t, l}^{(\tau - 1)} \|_2 \leq \epsilon_{\textrm{opt}}}$ for iteration ${\tau}$:
	\begin{equation}
		\bx_{t, l} = \stackrel{\local}{\argmax}_{\bx \in \calX} a(\bx | \calD_{t-1}).
		\label{eqn:local_acq}
	\end{equation}
         $\bx_{t, l}$ is referred to as a local optimizer.
\end{defn}

% Definition 3
\begin{defn}[Multi-started local optimizer]
	\label{def:mslo_af}
	Suppose that $\{ \bx_{t, l_1}, \ldots, \bx_{t, l_N} \}$ is a set of $N$ local optimizers, each of which is
	determined by a local optimization method \eqref{eqn:local_acq}, starting from a different initial condition.
	The multi-started local optimizer, denoted by $\bx_{t, m}$, is the one at which $a(\bx | \calD_{t-1})$
	achieves the maximum:
	\begin{equation}
		\bx_{t, m} = \stackrel{\textrm{m-local}}{\argmax}_{\bx \in \calX} a(\bx | \calD_{t-1}).
		\label{eqn:multi_local_acq}
	\end{equation}
\end{defn}

With solutions to \eqref{eqn:aqmax}, defined in \eqref{eqn:global_acq}, \eqref{eqn:local_acq},
and \eqref{eqn:multi_local_acq}, we define {\em instantaneous regret} for each of these solutions
and {\em instantaneous regret difference} for each of local solutions
below.

% Definition 4
\begin{defn}[Instantaneous regret]
	\label{def:instant_regret}
	Suppose that $\bx^{\dagger}$ is the true global minimum of the objective function in \eqref{eqn:naive_go}.
	Denote by $\bx_t$ a maximum of acquisition function $a(\bx | \calD_{t-1})$ at round $t$, determined by either
	a global or local optimization method.
	The instantaneous regret $r_t $ at round $t$ is defined as
	\begin{equation}
		r_t = f(\bx_t) - f(\bx^{\dagger}).
		\label{eqn:instant_regret}
	\end{equation}
	Depending on an optimization method (i.e., one of global, local, and multi-started local optimization methods) used 
	to search for a maximum of the acquisition function, we define the following instantaneous regrets:
	$r_{t, g} = f(\bx_{t, g}) - f(\bx^\dagger)$, 
	$r_{t, l} = f(\bx_{t, l}) - f(\bx^\dagger)$, and 
	$r_{t, m} = f(\bx_{t, m}) - f(\bx^\dagger)$.
\end{defn}

% Definition 5
\begin{defn}[Instantaneous regret difference]
With \defref{def:instant_regret}, we define 
instantaneous regret differences for an local optimizer $\bx_{t, l}$ 
and for a multi-started local optimizer $\bx_{t, m}$:
\begin{align}
\label{eqn:rdtl}
\left| r_{t, g} - r_{t, l} \right| &= \left|  f(\bx_{t, g}) - f(\bx_{t, l}) \right|, \\
\label{eqn:rdtm}
\left| r_{t, g} - r_{t, m} \right| &= \left|  f(\bx_{t, g}) - f(\bx_{t, m}) \right|,
\end{align}
which measures a performance gap with respect to the one induced by $\bx_{t, g}$, at round $t$.
\end{defn}

Henceforth, instantaneous regret and instantaneous regret difference 
are simply called to \emph{regret} and \emph{regret difference}, respectively.\footnote{In Bayesian optimization, a cumulative regret is usually used to analyze the performance of convergence quality.
By Lemma~5.4 of \cite{SrinivasN2010icml} and Theorem~3 of \cite{ChowdhurySR2017icml}, 
our analysis on instantaneous regret differences can be expanded into the analysis on cumulative regrets.
However, to concentrate the scope of this work on the behavior of local optimizers, these analyses are not included in this paper.}

\subsection{Maximization of Acquisition Functions \label{subsec:afo}}

As described earlier, we may consider either global or local solutions to \eqref{eqn:aqmax}.
Famous global optimization methods include 
DIRECT~\cite{JonesDR1993jota} and CMA-ES~\cite{HansenN2016arxiv}.
DIRECT is a deterministic Lipschitzian-based derivative-free partitioning method where it observes function values 
at the centers of rectangles and divides the rectangles without the Lipschitz constant iteratively.
CMA-ES is a stochastic derivative-free method based on evolutionary computing. 
In this paper we use DIRECT to determine $\bx_{t,g}$ in \eqref{eqn:global_acq}. 

A local optimization method we used to determine $\bx_{t,l}$ or $\bx_{t,m}$, given in \eqref{eqn:local_acq}
or \eqref{eqn:multi_local_acq} is the Broyden-Fletcher-Goldfarb-Shanno (BFGS) algorithm,
which is a quasi-Newton optimization technique.
A limited memory version, referred to as L-BFGS~\cite{LiuDC1989mp} 
and a constrained version known as L-BFGS-B are widely used 
in the Bayesian optimization literature~\cite{PichenyV2016neurips,WangZ2018aistats,FrazierPI2018arxiv}. 
Multi-started optimization methods are also widely used \cite{BrochuE2010arxiv,HutterF2011lion},
where a local optimization method starts from $N$ distinct initializations and
such $N$ local solutions started from $N$ distinct initializations are combined to determine the best local solution.

Compared to our work, \cite{WilsonJT2018neurips} introduces a reparameterization form 
to allow differentiability of Monte Carlo acquisition functions to integrate them and query in parallel,
which is not related to the topics covered in this paper.

\section{Performance Analysis\label{sec:theo_analysis}}

In this section we present our main contribution on the performance analysis for
the local optimizer and the multi-started local optimizer,
given in \eqref{eqn:local_acq} and \eqref{eqn:multi_local_acq}.

\subsection{Main Theorems\label{subsec:main_theorems}}

Before introducing the lemmas used to prove the main theorems, 
we explain the main theorems and their intuition first.
Our theorems are described as follows.

\begin{thm}
	\label{thm:first}
	Given $\delta_{l} \in [0, 1)$ and ${\epsilon_l, \epsilon_1, \epsilon_2 > 0}$, 
	the regret difference for a local optimizer $\bx_{t, l}$ 
	at round $t$, $\left| r_{t, g} - r_{t, l} \right|$ is less than ${\epsilon_l}$ with a probability at least ${1 - \delta_{l}}$:
	\begin{equation}
		\bbP \big( \left| r_{t, g} - r_{t, l} \right| < \epsilon_{l} \big) \geq  1- \delta_{l},
	\end{equation}
	where
	$\delta_{l} = \frac{\gamma}{\epsilon_1} (1 - \beta_g) + \frac{M}{\epsilon_2}$,
	$\epsilon_{l} = \epsilon_1 \epsilon_2$,
	$\gamma = \max_{\bx_i, \bx_j \in \calX} \| \bx_i - \bx_j \|_2$ is the size of $\calX$,
	$\beta_g$ is the probability that a local optimizer of the acquisition function collapses with its global optimizer,
	and $M$ is the Lipschitz constant explained in \lemref{lem:lipschitz}.
\end{thm}

\thmref{thm:first} is extended for a multi-started local optimizer.
\begin{thm}
	\label{thm:second}
	Given $\delta_{m} \in [0, 1)$ and ${\epsilon_{m}, \epsilon_2, \epsilon_3 > 0}$, 
	a regret difference for a multi-started local optimizer $\bx_{t, m}$, 
	determined by
	starting from $N$ initial points at round $t$,
	is less than $\epsilon_m$ with a probability at least $1 - \delta_{m}$:
	\begin{equation}
		\bbP \big( \left| r_{t, g} - r_{t, m} \right| < \epsilon_{m} \big) \geq 1 - \delta_{m},
		\label{eqn:thm_2_multi_start}
	\end{equation}
	where 
	$\delta_{m} = \frac{\gamma}{\epsilon_3} \left( 1 - \beta_g \right)^N + \frac{M}{\epsilon_2}$,
	${\epsilon_{m} = \epsilon_2 \epsilon_3}$,
	$\gamma = \max_{\bx_i, \bx_j \in \calX} \| \bx_i - \bx_j \|_2$ is the size of $\calX$,
	$\beta_g$ is the probability that a local optimizer of the acquisition function collapses with its global optimizer,
	and $M$ is the Lipschitz constant explained in \lemref{lem:lipschitz}.
\end{thm}

As shown in \thmref{thm:first}, $\left| r_{t, g} - r_{t, l} \right|$ is smaller than $\epsilon_l$ with a probability $1 - \delta_l$.
It implies 
the probability $1- \delta_l$ is controlled by three statements related to $\gamma$, $\beta_g$, and $M$: 
the probability is decreased 
(i) as $\gamma$ is increased, 
(ii) as $\beta_g$ is decreased, 
and (iii) as $M$ is increased.
If $\calX$ is a relatively small space, $\gamma$ is naturally small.
Moreover, $\beta_g$ is close to one 
if converging to global optimum by \defref{def:mslo_af} is relatively easy for some reasons: 
(i) a small number of local optima exist, 
or (ii) a global optimum is easily reachable.

\thmref{thm:second} suggests the implications that are similar with \thmref{thm:first} in terms of the control factors of $1 - \delta_m$.
The main difference of two theorems is that $\delta_m$ is related to the number of initial points in \defref{def:mslo_af}, $N$.
Because $0 \leq 1 - \beta_g < 1$ is given, 
$N$ can control the bound of \eqref{eqn:thm_2_multi_start}.
Additionally, by this difference, we theoretically reveal how many runs for a multi-started local optimizer 
are needed to obtain the sufficiently small regret difference over a global optimizer.

\subsection{Lemmas}

Next, we prove two statements (i) how different the global and local optimizers are 
(see \lemref{lem:pi_lipschitz} to \lemref{lem:bounded_distance}), 
and (ii) how steep the slope between the global and local optimizers is
(see \lemref{lem:lipschitz}).
First, the Lipschitz continuities of acquisition functions are proved in the subsequent lemmas.

% Lemma 1
\begin{lem}[Lipschitz continuity of PI]
	\label{lem:pi_lipschitz}
	The PI criterion $a(\bx|\calD_{t-1})$, formed by the posterior distribution calculated by 
	Gaussian process regression on $\calD_{t-1}$ is Lipschitz-continuous.
\end{lem}

\begin{proof}
	$\calX$ is a compact subset of $d$-dimensional space $\bbR^d$.
	In this paper, we analyze our theorem with Gaussian process regression as a surrogate function.
	If we are given $t - 1$ covariates $\bX = [\bx_1 \cdots \bx_{t - 1}]^\top$ obtained from $\calX$ and their corresponding responses $\by = [y_1 \cdots y_{t - 1}] \in \bbR^{t - 1}$, 
	posterior mean and variance functions, $\mu(\bx)$ and $\sigma^2(\bx)$ over $\bx \in \calX$ can be computed, using Gaussian process regression~\cite{RasmussenCE2006book}:
	\begin{align}
		\mu(\bx) &= \bk(\bx, \bX) \tilde{\bK}^{-1} \by, \label{eqn:p_mean}\\
		\sigma^2(\bx) &= k(\bx, \bx) - \bk(\bx, \bX) \tilde{\bK}^{-1} \bk(\bX, \bx), \label{eqn:p_var}
	\end{align}
	where $k(\cdot, \cdot)$ is a covariance function, $\tilde{\bK} = \bK(\bX, \bX) + \sigma_n^2 \bI$, 
	and $\sigma_n$ is an observation noise.
	$\bk(\cdot, \cdot)$ accepts a vector and a matrix as two arguments 
	(e.g., $\bk(\bx, \bX) = [k(\bx, \bx_1) \cdots k(\bx, \bx_{t - 1})]$).
	Similarly, $\bK(\cdot, \cdot)$ can take two matrices (e.g., $\bK(\bX, \bX) = [\bk(\bX, \bx_1) \cdots \bk(\bX, \bx_{t - 1})]$).
	Before showing the Lipschitz continuity of the acquisition function, we first show the derivatives of \eqref{eqn:p_mean} and \eqref{eqn:p_var}.
	It depends on the differentiability of covariance functions, but the famous covariance functions, 
	which are used in Bayesian optimization are usually at least once differentiable 
	(e.g., squared exponential kernel\footnote{Squared exponential kernel is infinite times differentiable.} and Mat\'{e}rn kernel\footnote{Mat\'{e}rn kernel is $\lceil \nu \rceil - 1$ times differentiable.}).
	Thus, the derivatives of \eqref{eqn:p_mean} and \eqref{eqn:p_var} are
	\begin{align}
		\frac{\partial \mu(\bx)}{\partial \bx} &= \frac{\partial \bk(\bx, \bX)}{\partial \bx} \tilde{\bK}^{-1} \by, \label{eqn:p_p_mean}\\
		\frac{\partial \sigma^2(\bx)}{\partial \bx} &= - 2 \frac{\partial \bk(\bx, \bX)}{\partial \bx} \tilde{\bK}^{-1} \bk(\bX, \bx), \label{eqn:p_p_var}
	\end{align}
	using vector calculus identities.
	To show \eqref{eqn:p_p_mean} and \eqref{eqn:p_p_var} are bounded, 
	each term in both equations should be bounded.
	For all $i \in \{1, \ldots, t - 1\}$, 
	$\by$ and $\bk(\bX, \bx)$ are obviously bounded:
	\begin{equation}
		| y_i | < \infty \quad \textrm{and} \quad | k(\bx_i, \bx) | < \infty, \label{eqn:bounds_y_k_xi_x}
	\end{equation}
	but $\partial \bk(\bx, \bX) / \partial \bx$ and $\tilde{\bK}^{-1}$ should be revealed.
	
	First of all, the bound of $\partial \bk(\bx, \bX) / \partial \bx$ would be proved in \lemref{lem:bounds_k}.
	For the latter one, 
	all entries of $\tilde{\bK}^{-1}$ are bounded by the Kantorovich and Wielandt inequalities~\cite{RobinsonPD1992jna}.
	As a result, by \eqref{eqn:bounds_y_k_xi_x}, \lemref{lem:bounds_k}, and the Kantorovich and Wielandt inequalities, 
	the following inequalities are satisfied:
	\begin{equation}
		\left| \frac{\partial \mu(\bx)}{\partial x_i} \right| < \infty \quad \textrm{and} \quad
		\left| \frac{\partial \sigma^2(\bx)}{\partial x_i} \right| < \infty,
	\end{equation}
for all $i \in \{1, \ldots, d\}$.
Thus, we can say
\begin{equation}
	\left| \frac{\partial \mu(\bx)}{\partial x_i} \right| < M_{\mu} \quad \textrm{and} \quad
	\left| \frac{\partial \sigma^2(\bx)}{\partial x_i} \right| < M_{\sigma^2}, \label{eqn:mu_var_lip_m}
\end{equation}
for some $M_{\mu}, M_{\sigma^2} < \infty$.
It implies that $\mu(\bx)$ and $\sigma^2(\bx)$ are Lipschitz-continuous with the Lipschitz constants $M_{\mu}$ and $M_{\sigma^2}$ to each axis direction, respectively.
Thus, \eqref{eqn:p_mean} and \eqref{eqn:p_var} are Lipschitz-continuous 
with the Lipschitz constants $d M_{\mu}$ and $d M_{\sigma^2}$, where $d$ is a dimensionality of $\bx$,
by the triangle inequality.

	PI is written with 
	$z(\bx) = \left(f(\bx^\ddagger) - \mu(\bx) \right) / \sigma(\bx)$ if $\sigma(\bx) > \sigma_n$, and $0$ otherwise, 
	where $\bx^\ddagger$ is the current best observation which has a minimum of $\by$.
	Given the PI criterion
	$a_{\textrm{PI}}(\bx) = \Phi(z(\bx))$,
	where $\Phi(\cdot)$ is a cumulative distribution function of standard normal distribution,
	the derivative of PI criterion is
	\begin{equation}
		\frac{\partial a_{\textrm{PI}}(\bx)}{\partial \bx} = \phi(z(\bx)) \frac{\partial z(\bx)}{\partial \bx} 
		= \phi(z(\bx)) \left( \frac{\mu(\bx) - f(\bx^\ddagger)}{\sigma^2(\bx)} \frac{\partial \sigma(\bx)}{\partial \bx} - \frac{1}{\sigma(\bx)} \frac{\partial \mu(\bx)}{\partial \bx} \right), \label{eqn:d_pi}
	\end{equation}
	where $\phi(\cdot)$ is a probability density function of standard normal distribution.
	By \eqref{eqn:p_mean}, \eqref{eqn:p_var}, and \eqref{eqn:mu_var_lip_m}, 
	we can show \eqref{eqn:d_pi} is bounded, $\left\| \partial a_{\textrm{PI}}(\bx)/\partial \bx \right\|_2 < \infty$.
	\qed
\end{proof}

% Lemma 2
\begin{lem}[Lipschitz continuity of EI]
	\label{lem:ei_lipschitz}
	The EI criterion $a(\bx|\calD_{t-1})$, formed by the posterior distribution calculated by 
	Gaussian process regression on $\calD_{t-1}$ is Lipschitz-continuous.
\end{lem}

\begin{proof}
	EI expresses with $z(\bx) = \left(f(\bx^\ddagger) - \mu(\bx) \right) / \sigma(\bx)$ if $\sigma(\bx) > \sigma_n$, and $0$ otherwise, 
where $\bx^\ddagger$ is the current best observation which has a minimum of $\by$.
	For the EI criterion:
	\begin{equation}
		a_{\textrm{EI}}(\bx) = \left( f(\bx^\ddagger) - \mu(\bx) \right) \Phi(z(\bx)) + \sigma(\bx) \phi(z(\bx)), \label{eqn:ei}
	\end{equation}
	the derivative of \eqref{eqn:ei} is
	\begin{align}
		\frac{\partial a_{\textrm{EI}}(\bx)}{\partial \bx} &=
		\left(f(\bx^\ddagger) - \mu(\bx) \right) \phi(z(\bx)) \frac{\partial z(\bx)}{\partial \bx}
		- \frac{\partial \mu(\bx)}{\partial \bx} \Phi(z(\bx))
		\nonumber\\
		&\quad
		+ \sigma(\bx) \phi'(z(\bx)) \frac{\partial z(\bx)}{\partial \bx}
		+ \frac{\partial \sigma(\bx)}{\partial \bx} \phi(z(\bx)). \label{eqn:d_ei}
	\end{align}
	
	Similar to \eqref{eqn:d_pi}, the following inequality, 
	\begin{equation}
		\left\| \frac{\partial a_{\textrm{EI}}(\bx)}{\partial \bx} \right\|_2 < \infty, 
	\end{equation}
	is satisfied.
	\qed
\end{proof}

% Lemma 3
\begin{lem}[Lipschitz continuity of GP-UCB]
	\label{lem:ucb_lipschitz}
	GP-UCB $a(\bx|\calD_{t-1})$, formed by the posterior distribution calculated by 
	Gaussian process regression on $\calD_{t-1}$ is Lipschitz-continuous.
\end{lem}
\begin{proof}
	GP-UCB~\cite{SrinivasN2010icml} and its derivative are
	\begin{align}
		a_{\textrm{UCB}}(\bx) &= - \mu(\bx) + \alpha \sigma(\bx), \label{eqn:ucb}\\
		\frac{\partial a_{\textrm{UCB}}(\bx)}{\partial \bx} &= - \frac{\partial \mu(\bx)}{\partial \bx} + \alpha \frac{\partial \sigma(\bx)}{\partial \bx}, \label{eqn:d_ucb}
	\end{align}
	where $\alpha$ is a coefficient for balancing exploration and exploitation.
	By \eqref{eqn:mu_var_lip_m}, the following inequality, 
	\begin{align}
		\left| \frac{\partial a_{\textrm{UCB}}(\bx)}{\partial x_i} \right| &= \left| - \frac{\partial \mu(\bx)}{\partial x_i} + \alpha \frac{\partial \sigma(\bx)}{\partial x_i} \right|
		\leq \left| - \frac{\partial \mu(\bx)}{\partial x_i} \right| + \alpha \left| \frac{\partial \sigma(\bx)}{\partial x_i} \right|\nonumber\\
		&= \left| \frac{\partial \mu(\bx)}{\partial x_i} \right| + \alpha \left| \frac{\partial \sigma(\bx)}{\partial x_i} \right|
		\leq M_{\mu} + \alpha \sqrt{M_{\sigma^2}},
	\end{align}
	is bounded for $i \in \{1, \ldots, d\}$.
	Therefore, \eqref{eqn:d_ucb} is bounded.
	\qed
\end{proof}

% Lemma 4
\begin{lem}
	\label{lem:bounds_k}
	Given a stationary covariance function $k(\cdot, \cdot)$ that is widely used in Gaussian process regression~\cite{WilsonAG2013icml,DuvenaudD2014thesis}, 
	$\partial \bk(\bx, \bX) / \partial \bx$ is bounded where 
	$\bX \in \bbR^{n \times d}$ and
	$\bk(\bx, \bX) = [k(\bx, \bx_1) \cdots k(\bx, \bx_n)]$.
\end{lem}

\begin{proof}
The well-known stationary covariance functions such as squared exponential (SE) 
and Mat\`ern kernels are utilized in Gaussian process regression~\cite{WilsonAG2013icml,DuvenaudD2014thesis}.
Since such kernels are additive or multiplicative~\cite{DuvenaudD2014thesis}, 
this lemma can be generalized to most of kernels applied in Gaussian process regression.
In this paper, we analyze the cases of SE and Mat\'ern 5/2 kernels.
Because the cases of Mat\'ern 3/2 and periodic kernels can be simply extended from the cases analyzed, it is omitted.
The SE and Mat\`ern 5/2 kernels are at least one time differentiable, thus $\partial \bk(\bx, \bX) / \partial \bx$ can be computed.
Before explaining in detail,
$\partial \bk(\bx, \bX) / \partial \bx$ is written as
\begin{equation}
	\frac{\partial \bk(\bx, \bX)}{\partial \bx} = \left[ \frac{\partial k(\bx, \bx_1)}{\partial \bx} \cdots \frac{\partial k(\bx, \bx_n)}{\partial \bx} \right],
	\label{eqn:p_k_x_X}
\end{equation}
for $\bX = [\bx_1 \cdots \bx_n]$.
Furthermore, we can define
\begin{equation}
	d(\bx_1, \bx_2) = \sqrt{(\bx_1 - \bx_2)^\top \bL^{-1} (\bx_1 - \bx_2)},
	\label{eqn:dist}
\end{equation}
where $\bL$ is a diagonal matrix of which entries are lengthscales for each each dimension.
For simplicity, 
$\bx_1 - \bx_2$ is denoted as $\bs_{12}$.
The derivative of \eqref{eqn:dist} is
\begin{equation}
	\frac{\partial d(\bx_1, \bx_2)}{\partial \bx_1} = \left(\bL^{-1} \bs_{12} \right) \left(\bs_{12}^\top \bL^{-1} \bs_{12}\right)^{-\frac{1}{2}}.
	\label{eqn:d_dist}
\end{equation}

The derivative of $d^2(\bx_1, \bx_2)$ is
$\frac{\partial d^2(\bx_1, \bx_2)}{\partial \bx_1} = 2 \bL^{-1} \bs_{12}$.

First, the SE kernel is
$k(\bx_1, \bx_2) = \sigma_s^2 \exp ( -\frac{1}{2} d^2(\bx_1, \bx_2) )$
where $\sigma_s$ is a signal scale.
The derivative of each $\partial k(\bx, \bx_i)/\partial \bx$ is
\begin{align}
	\frac{\partial k(\bx, \bx_i)}{\partial \bx}
	&= \frac{\partial}{\partial \bx} \left( \sigma_s^2 \exp \left( -\frac{1}{2} d^2(\bx, \bx_i) \right) \right)
	= k(\bx, \bx_i) \frac{\partial}{\partial \bx} \left( -\frac{1}{2} d^2(\bx_1, \bx_2) \right)\nonumber\\
	&= -\frac{k(\bx, \bx_i)}{2} \left( 2\bL^{-1} (\bx - \bx_i) \right)
	= -k(\bx, \bx_i) \left(\bL^{-1} (\bx - \bx_i) \right). \label{eqn:se_k_x_xi_x}
\end{align}

Because all the terms of \eqref{eqn:se_k_x_xi_x} are bounded in $\calX$, 
$\| \partial k(\bx, \bx_i)/\partial \bx \|_2 < \infty$
is satisfied for all $i = \{1, \ldots, n\}$.
Note that \eqref{eqn:dist} and \eqref{eqn:d_dist} are bounded, because $\bx_1, \bx_2 \in \calX$.

The Mat\'ern $5/2$ kernel is
\begin{equation}
k(\bx_1, \bx_2) = \sigma_s^2 \left( 1 + \sqrt{5}d(\bx_1, \bx_2) + \frac{5}{3} d^2(\bx_1, \bx_2) )
( -\sqrt{5}d(\bx_1, \bx_2) \right),
\end{equation}
where $\sigma_s$ is a signal scale.
Its derivative is
\begin{equation}
\frac{\partial k(\bx, \bx_i)}{\partial \bx}
= -\frac{5\sigma_s^2}{3} \left(1 + \sqrt{5} d(\bx, \bx_i) \right)
\exp ( -\sqrt{5}d(\bx, \bx_i) ) \bL^{-1} (\bx - \bx_i).
\end{equation}
Since the derivative is bounded,
$\| \partial \bk(\bx, \bX)/\partial \bx \|_2 < \infty$ is satisfied.
Similarly, other kernels can be straightforwardly proved.
Thus, this lemma is concluded.
\qed
\end{proof}

From now, we show the number of local optima is upper-bounded, using the condition involved in a frequency domain.

% Lemma 5
\begin{lem}
	\label{lem:spectral}
	Let $\calX \subset \bbR^d$ be a compact set.
	Given some sufficiently large $|\hat{\bxi}| > 0$, 
	a spectral density of stationary covariance function for Gaussian process regression is zero for all $|\bxi| > |\hat{\bxi}|$ with very high probability.
	Then, the number of local maxima at iteration $t$, $\rho_t$ is upper-bounded.
\end{lem}

\begin{proof}
	By Sard's theorem~\cite{SardA1942bams} for a Lipschitz-continuous function~\cite{BarbetL2016ijm}, critical points (i.e., the points whose gradients are zero) do not exist almost everywhere.
	Since the number of local maxima $\rho_t$ is upper-bounded by the number of critical points, it can be a starting point to bound $\rho_t$.
	Especially, by \lemref{lem:pi_lipschitz}, \lemref{lem:ei_lipschitz}, and \lemref{lem:ucb_lipschitz}, the number of local maxima $\rho_t$ can be restrained in the compact set $\calX$.
	Since it cannot express the upper-bound of $\rho_t$, we transform a stationary covariance function using a Fourier transform and obtain the spectral density of each covariance function~\cite[Chapter 4]{RasmussenCE2006book}.
	Because a spectral density of stationary covariance function is naturally a light-tail function by Bochner's theorem~\cite{SteinM1999book}, 
	a spectral density of covariance function for Gaussian process regression is zero for all $|\bxi| > |\hat{\bxi}|$ with very high probability (i.e., exponentially saturated probability over $|\bxi|$ due to the form of stationary kernels~\cite[Chapter 4]{RasmussenCE2006book}), 
	given some sufficiently large $|\hat{\bxi}|$.
	Then, a function has finite local maxima, which implies that the number of local maxima at iteration $t$ is upper-bounded.
	\qed
\end{proof}

Based on \lemref{lem:spectral}, we can prove the ergodicity of local maxima 
that are able to 
be discovered by the local optimization method and
coincided with the local optimizers started from different initial points.

% Lemma 6
\begin{lem}
	\label{lem:prob_both}
	Let the number of local maxima of acquisition function at iteration $t$ be $\rho_t$.
	Since local optimizers which are started from some initial conditions $\in \calX$ are ergodic to all the local maxima,
	the probability of reaching each solution is $\beta_1, \ldots, \beta_{\rho_t} > 0$ such that $\Sigma_{i = 1}^{\rho_t} \beta_i = 1$.
\end{lem}

\begin{proof}
	If we start from some different initial conditions $\in \calX$, it is obvious that all the local solutions are reachable.
	Therefore, all the local optimizers are ergodic, and the probability of reaching each solution is larger than zero and they sum to one: $\sum_{i = 1}^{\rho_t} \beta_i = 1$,
	where $\beta_1, \ldots, \beta_{\rho_t} > 0$.
	\qed
\end{proof}

We now prove the distance between two points acquired by 
\defref{def:go_af} and \defref{def:lo_af} is bounded with a probability.

% Lemma 7
\begin{lem}
	\label{lem:bounded_distance}
	Let $\calX \subset \mathbb{R}^d$ be a compact space where 
	$\gamma = \max_{\bx_1, \bx_2 \in \calX} \| \bx_1 - \bx_2\|_2.$
	Then, for ${\gamma > \epsilon_1 > 0}$, we have
	\begin{equation}
		\bbP \big(\| \bx_{t, g} - \bx_{t, l} \|_2 \geq \epsilon_1 \big) 
		\leq \frac{\gamma}{\epsilon_1} (1 - \beta_g),
	\end{equation}
	where $\beta_g$ is the probability that some local optimizer is coincided with the global optimizer of the acquisition function.
\end{lem}

\begin{proof}
By Markov inequality for $\epsilon_1 > 0$,
we have
\begin{equation}
\label{eqn:lem_1_1}
\bbP \big(\| \bx_{t, g} - \bx_{t, l} \|_2 \geq \epsilon_1 \big) \leq \frac{1}{\epsilon_1} 
\bbE \big[ \| \bx_{t, g} - \bx_{t, l} \|_2 \big].
\end{equation}

Following from \lemref{lem:prob_both} and \eqref{eqn:lem_1_1},
the expectation in the right-hand side of \eqref{eqn:lem_1_1} is calculated as
\begin{align}
	&\frac{1}{\epsilon_1} \bbE \big[ \| \bx_{t, g} - \bx_{t, l} \|_2 \big] \nonumber\\
	&= \left. \frac{1}{\epsilon_1} \beta_g \| \bx_{t, g} - \bx_{t, l} \|_2  \right|_{\bx_{t, g} =  \bx_{t, l}}
	\quad + \left. \frac{1}{\epsilon_1} (1 - \beta_g) \| \bx_{t, g} - \bx_{t, l} \|_2  \right|_{\bx_{t, g} \neq  \bx_{t, l}} \nonumber\\
	&\leq \frac{1}{\epsilon_1} (1 - \beta_g) \gamma,
\end{align}
which completes the proof.
\qed
\end{proof}

The lower-bound of $\frac{| f(\bx_{t, 1}) - f(\bx_{t, 2}) |}{\| \bx_{t, 1} - \bx_{t, 2} \|_2}$ can be expressed with a probability as follows.

% Lemma 8
\begin{lem}
	\label{lem:lipschitz}
	Given any $\epsilon_2 > 0$, the probability that ${\frac{| f(\bx_{t, g}) - f(\bx_{t, l}) |}{\| \bx_{t, g} - \bx_{t, l} \|_2} \geq \epsilon_2}$ is less than $\frac{M}{\epsilon_2}$:
	\begin{equation}
		\bbP \left( \frac{| f(\bx_{t, g}) - f(\bx_{t, l}) |}{\| \bx_{t, g} - \bx_{t, l} \|_2} \geq \epsilon_2 \right) \leq \frac{M}{\epsilon_2}.
	\end{equation}
\end{lem}

\begin{proof}
	By Markov's inequality, 
	we can express
	\begin{equation}
		\bbP \left( \frac{| f(\bx_{t, g}) - f(\bx_{t, l}) |}{\| \bx_{t, g} - \bx_{t, l} \|_2} \geq \epsilon_2 \right)
		\leq \frac{1}{\epsilon_2} \bbE \left[ \frac{| f(\bx_{t, g}) - f(\bx_{t, l}) |}{\| \bx_{t, g} - \bx_{t, l} \|_2} \right] \leq \frac{M}{\epsilon_2},
	\end{equation}
	where $M$ is the Lipschitz constant of function $f$, because $f$ is $M$-Lipschitz continuous 
	and $\bx_{t, g}, \bx_{t, l} \in \calX$.
	\qed
\end{proof}

\subsection{Proof of \thmref{thm:first}}

Now we present the proof of \thmref{thm:first} here.

\begin{proof}
	The probability of $\left| r_{t, g} - r_{t, l} \right| < \epsilon_l$ can be written as
	\begin{align}
		\bbP \big( \left| r_{t, g} - r_{t, l} \right| < \epsilon_l \big)
		&= \bbP \big(| (f(\bx_{t, g}) - f(\bx^\dagger)) - (f(\bx_{t, l}) - f(\bx^\dagger)) | < \epsilon_l \big) \nonumber\\
		&= \bbP \left( \| \bx_{t, g} - \bx_{t, l} \|_2 \cdot \frac{| f(\bx_{t, g}) - f(\bx_{t, l}) |}{\| \bx_{t, g} - \bx_{t, l} \|_2} < \epsilon_l \right). \label{eqn:thm_1_1}
	\end{align}
	We define two events:
	\begin{equation}
		E_1 = \big( \| \bx_{t, g} - \bx_{t, l} \|_2 < \epsilon_1 \big) \quad \textrm{and} \quad
		E_2 = \left( \frac{| f(\bx_{t, g}) - f(\bx_{t, l}) |}{\| \bx_{t, g} - \bx_{t, l} \|_2} < \epsilon_2 \right).
	\end{equation}
	Then, \eqref{eqn:thm_1_1} can be expressed as
	\begin{equation}
		\bbP \big( \left| r_{t, g} - r_{t, l} \right| < \epsilon_l \big) = \bbP \big( E_1 \cap E_2 \big),
	\end{equation}
	where $\epsilon_l = \epsilon_1 \epsilon_2$.
	Thus, \eqref{eqn:thm_1_1} can be written as
	\begin{equation}
	\bbP \big( E_1 \cap E_2 \big) = 1 - \bbP \big( E_1^c \cup E_2^c \big) \geq 1 - \bbP \big( E_1^c \big) - \bbP \big( E_2^c \big),
	\end{equation}
	since $\bbP \big( E_1^c \cup E_2^c \big)  \leq \bbP \big( E_1^c \big) + \bbP \big( E_2^c \big)$ 
	by Boole's inequality. Then, we have
	\begin{align}
		\bbP \big( E_1 \cap E_2 \big) &\geq 1 - \bbP \big( \| \bx_{t, g} - \bx_{t, l} \|_2 \geq \epsilon_1 \big)
		- \bbP \left( \frac{| f(\bx_{t, g}) - f(\bx_{t, l}) |}{\| \bx_{t, g} - \bx_{t, l} \|_2} \geq \epsilon_2 \right)\nonumber\\
		&\geq 1 - \frac{\gamma}{\epsilon_1} (1 - \beta_g) - \frac{M}{\epsilon_2},
	\end{align}
	where \lemref{lem:bounded_distance} and \lemref{lem:lipschitz} are used to arrive at the last inequality.
	Therefore, the proof is completed:
	\begin{equation}
		\bbP \big( \left| r_{t, g} - r_{t, l} \right| < \epsilon_{l} \big) \geq 1- \delta_{l},
	\end{equation}
	where $\delta_{l} = \frac{\gamma}{\epsilon_1} (1 - \beta_g) + \frac{M}{\epsilon_2}$.
	\qed
\end{proof}

As described above, \thmref{thm:first} implies that the regret difference is basically controlled by $\gamma$, $\beta_g$, and $M$.
For example, if $\rho_t$ is close to one, the regret difference is tight with high probability.
On the other hand, if $\rho_t$ goes to infinity, the difference is tight with low probability.

\begin{figure}[t]
	\begin{center}
		\subfigure
		{
			\includegraphics[width=0.31\textwidth]{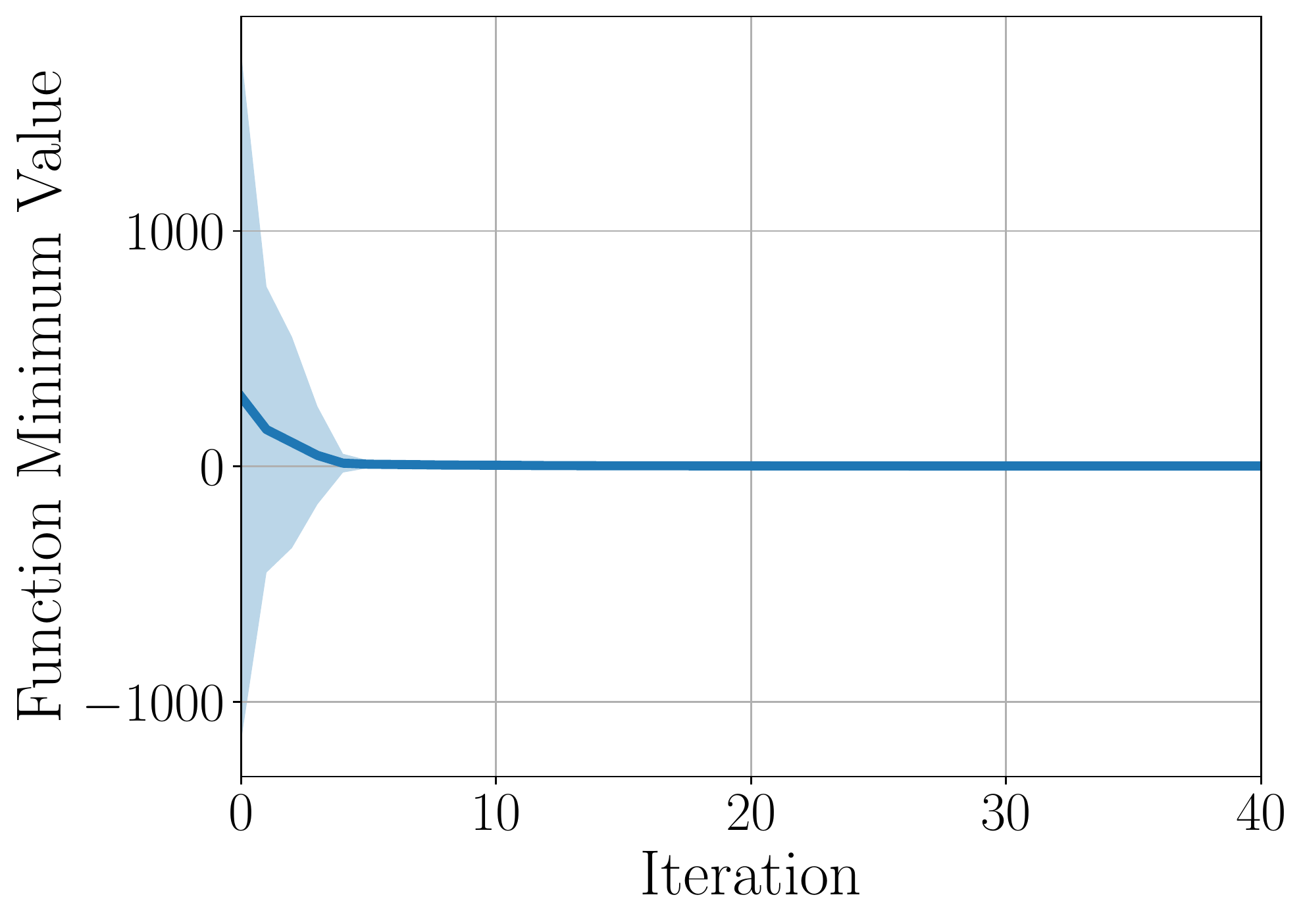}
			\label{fig:thm_101}
		}
		\subfigure
		{
			\includegraphics[width=0.31\textwidth]{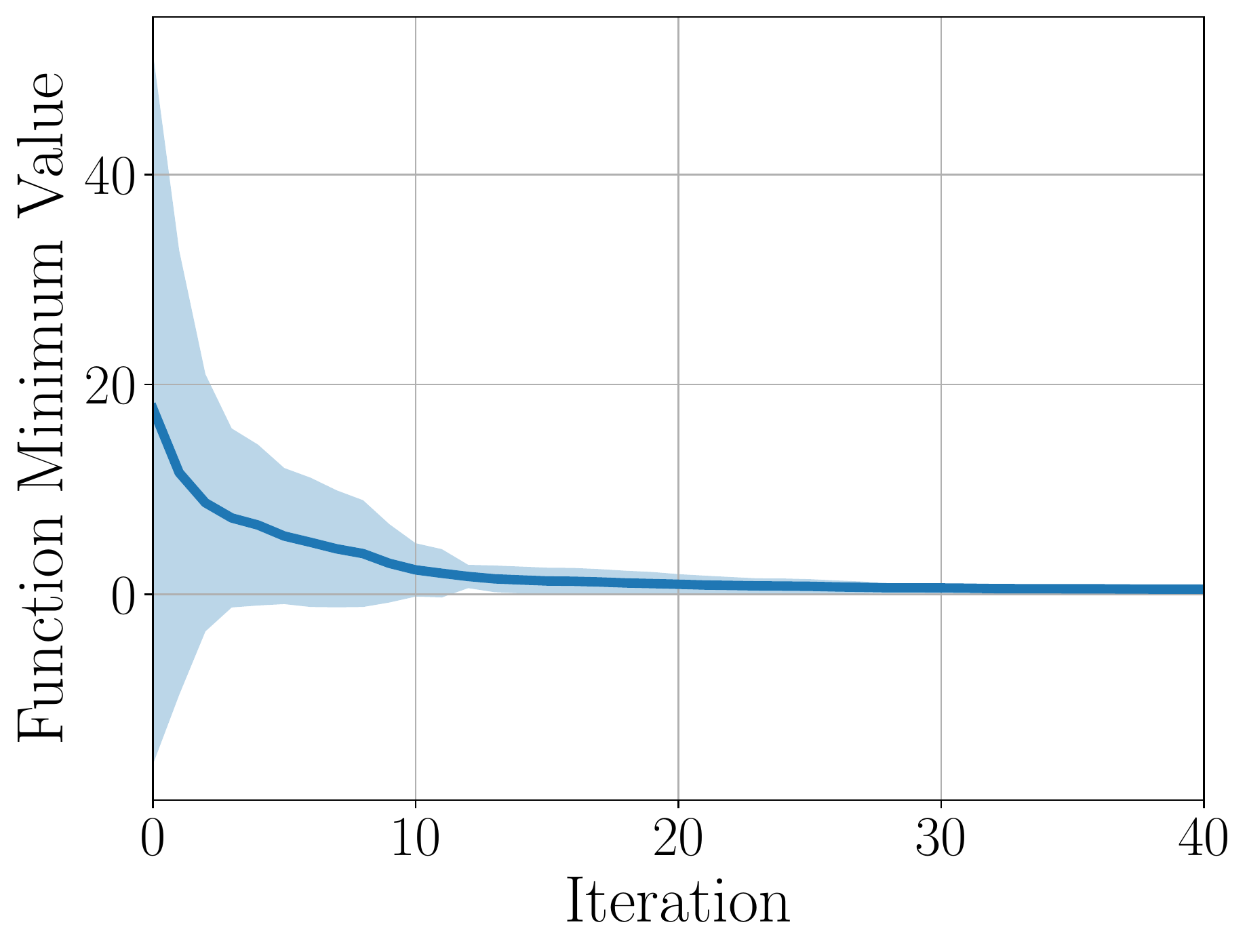}
			\label{fig:thm_102}
		}
		\subfigure
		{
			\includegraphics[width=0.31\textwidth]{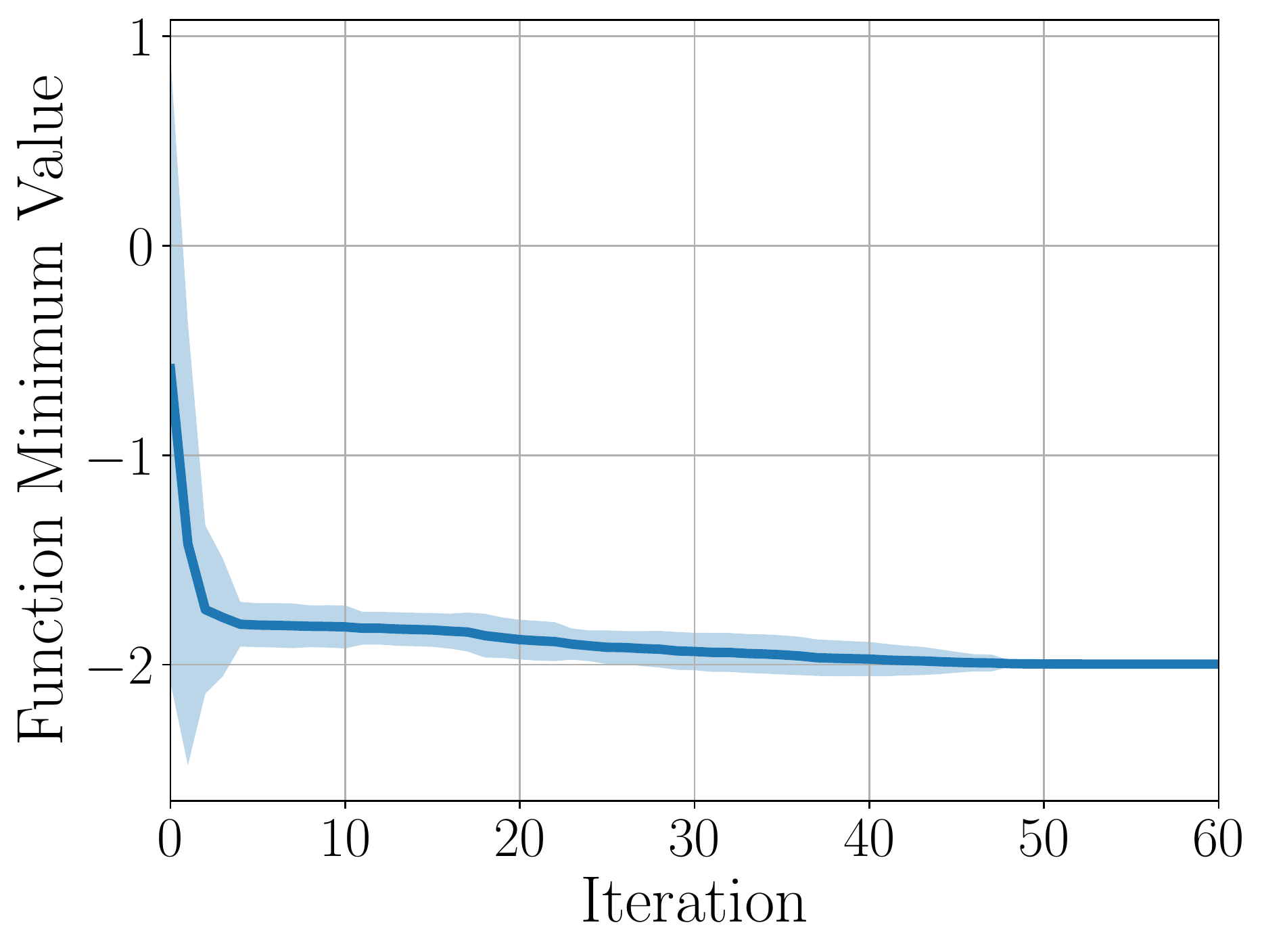}
			\label{fig:thm_103}
		}
		\setcounter{subfigure}{0}
		\subfigure[Beale]
		{
			\includegraphics[width=0.31\textwidth]{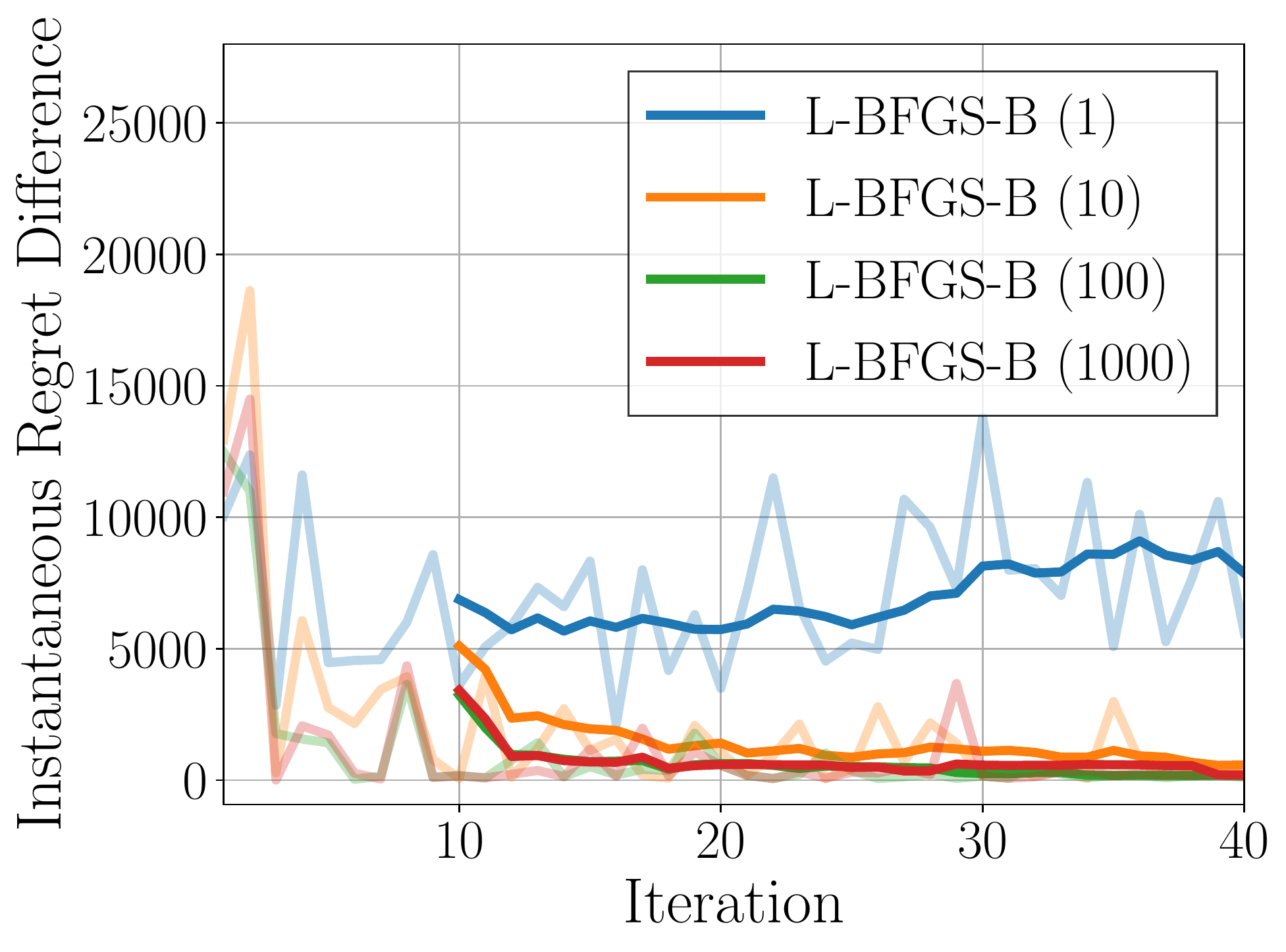}
			\label{fig:thm_111}
		}
		\subfigure[Branin]
		{
			\includegraphics[width=0.31\textwidth]{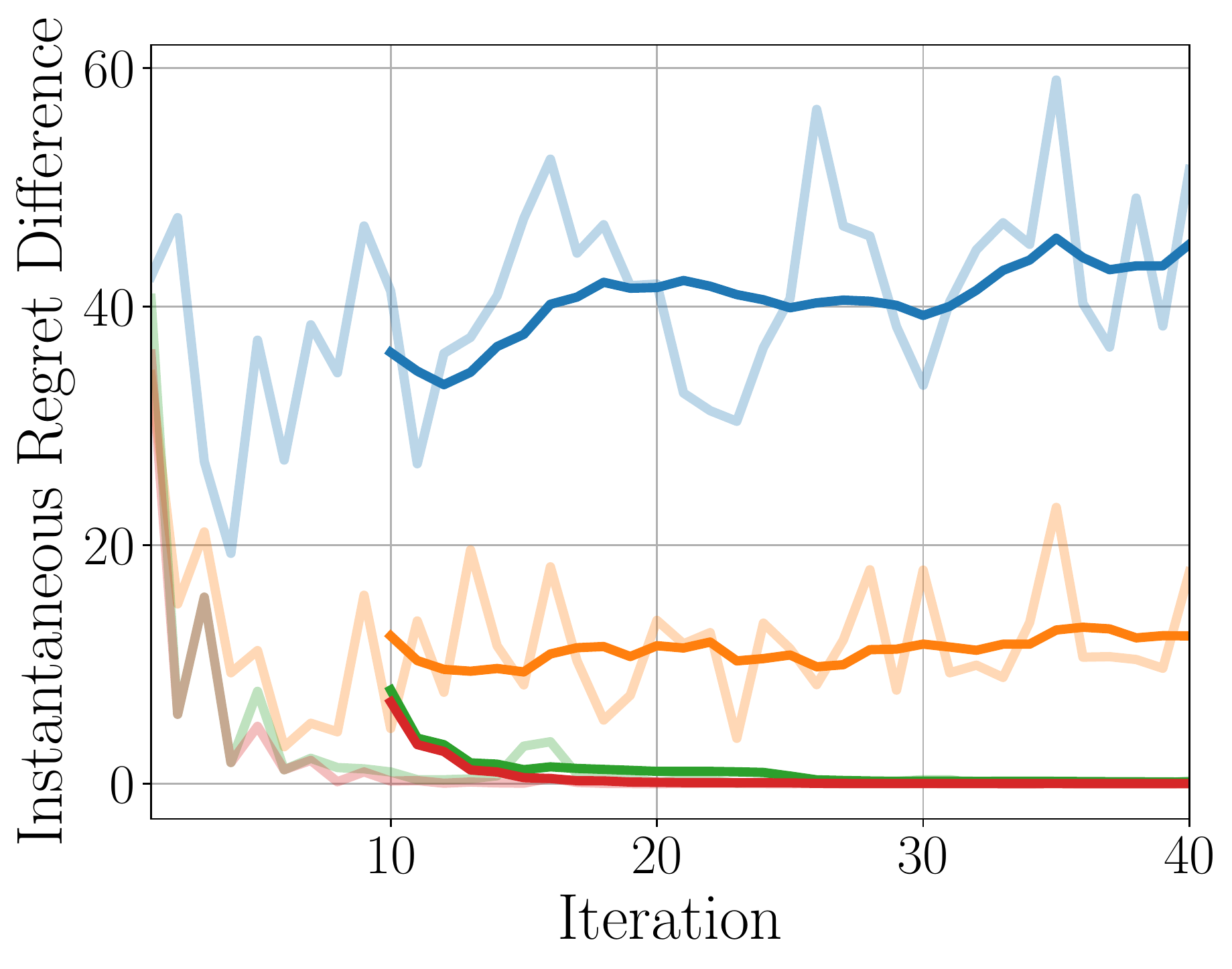}
			\label{fig:thm_112}
		}
		\subfigure[Cosines (2 dim.)]
		{
			\includegraphics[width=0.31\textwidth]{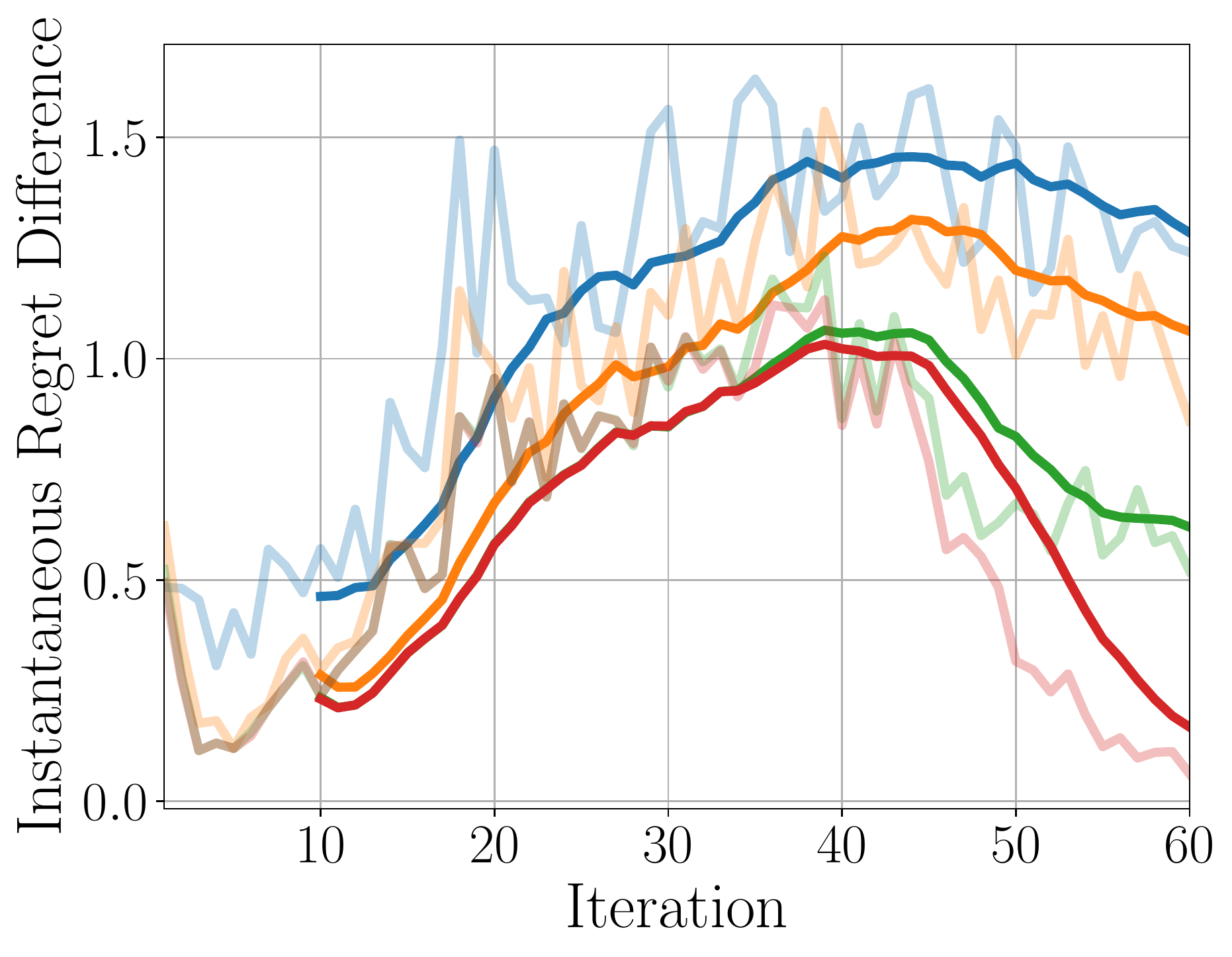}
			\label{fig:thm_113}
		}
	\end{center}
	\caption{Empirical results on \thmref{thm:first} and \thmref{thm:second}.
	The caption of each figure indicates a target function.
	The upper panel of each figure is an optimization result for a global optimizer,
	and the lower panel is the regret differences between a global optimizer and four types of local optimizer (i.e., multi-started local optimizers found by starting from \{1, 10, 100, 1000\} initial points).
	Some legends of the lower panels are missed not to interfere with the graphs, but all the legends are same.
	For the lower panels, transparent lines are observed instantaneous regret differences and solid lines are moving average (10 steps) of the transparent lines.
	All experiments are repeated 50 times.}
	\label{fig:val_thm_1}
\end{figure}

\begin{figure}[t]
	\begin{center}
		\subfigure
		{
			\includegraphics[width=0.31\textwidth]{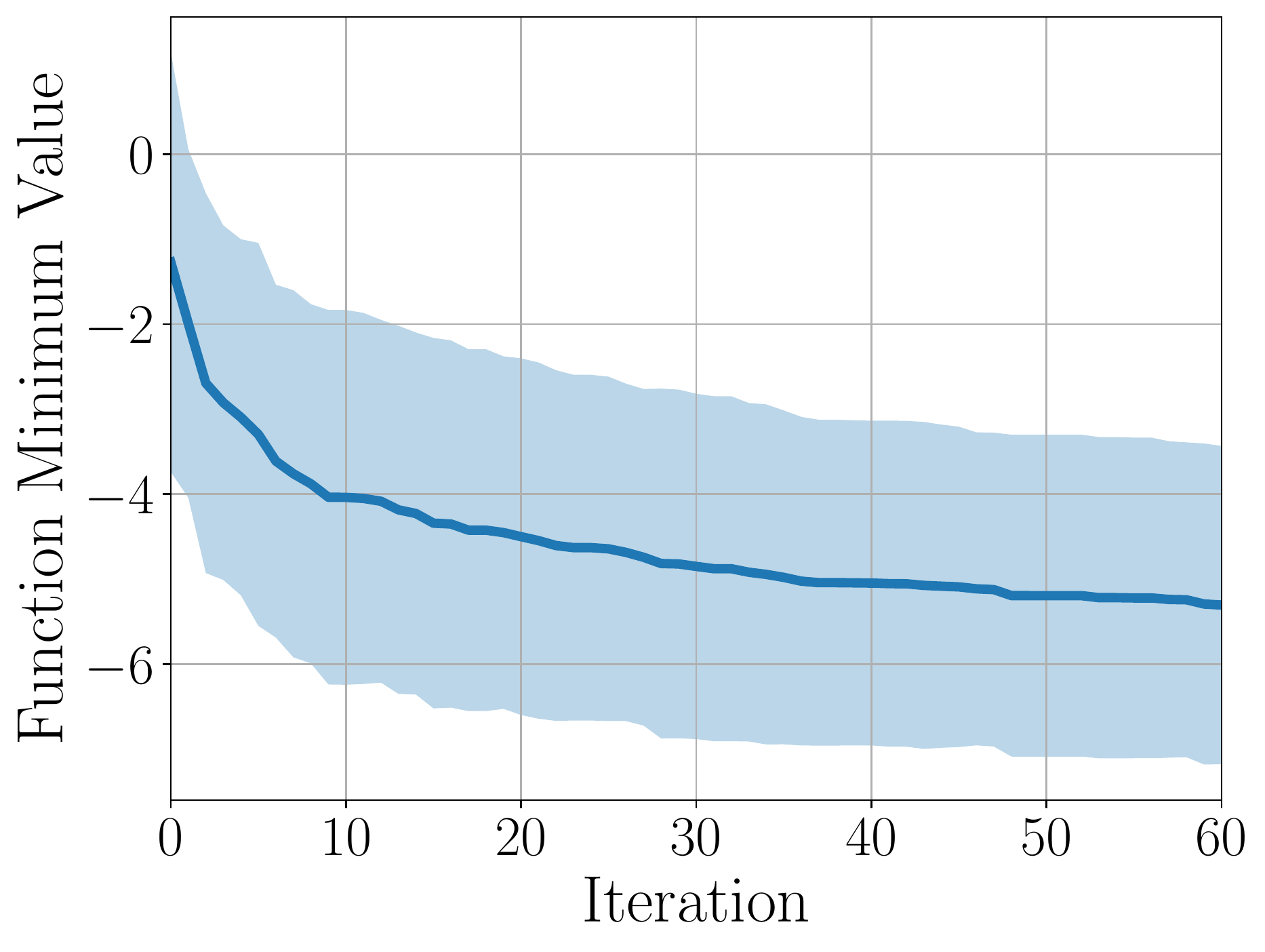}
			\label{fig:thm_201}
		}
		\subfigure
		{
			\includegraphics[width=0.31\textwidth]{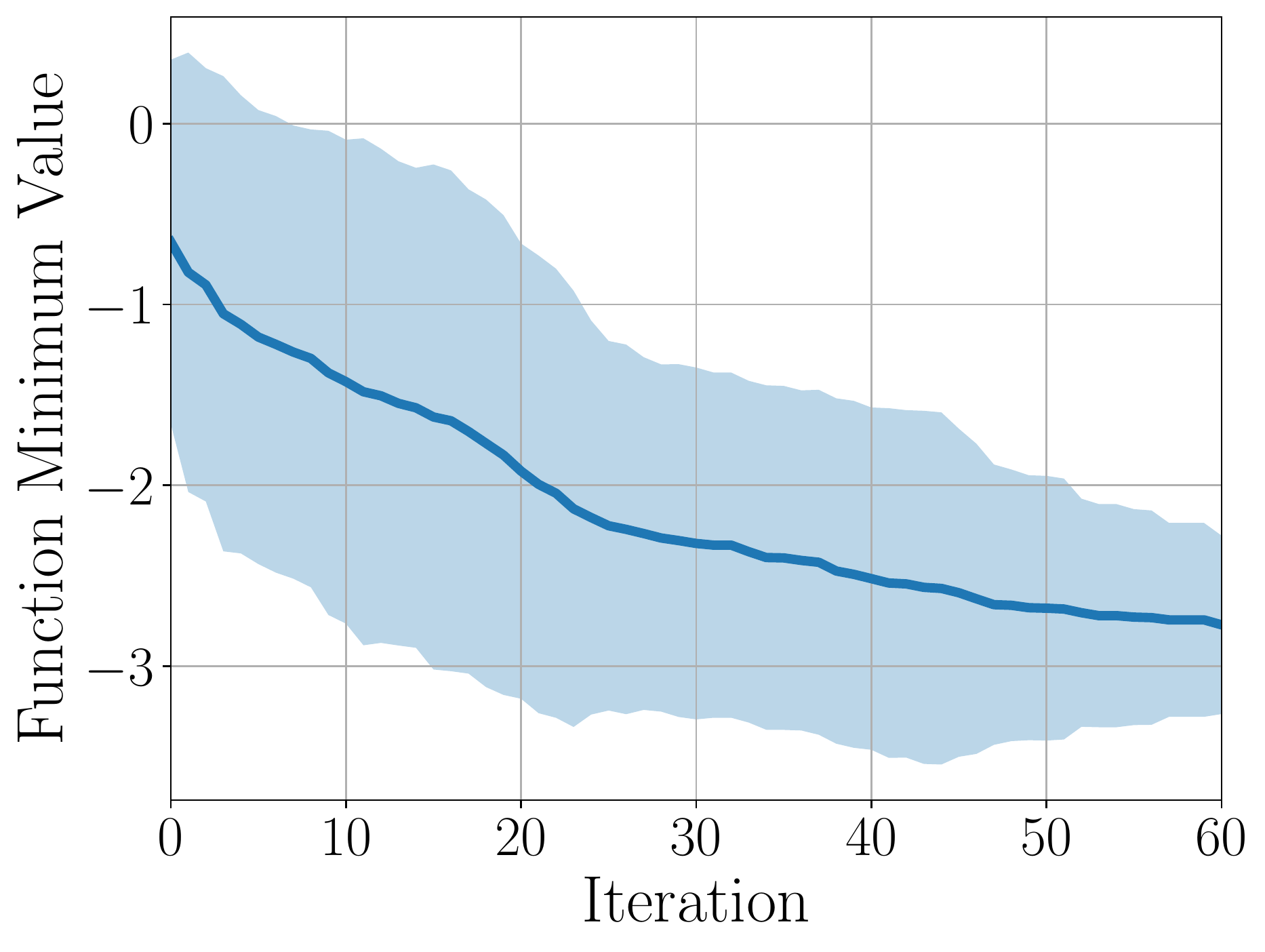}
			\label{fig:thm_202}
		}
		\subfigure
		{
			\includegraphics[width=0.31\textwidth]{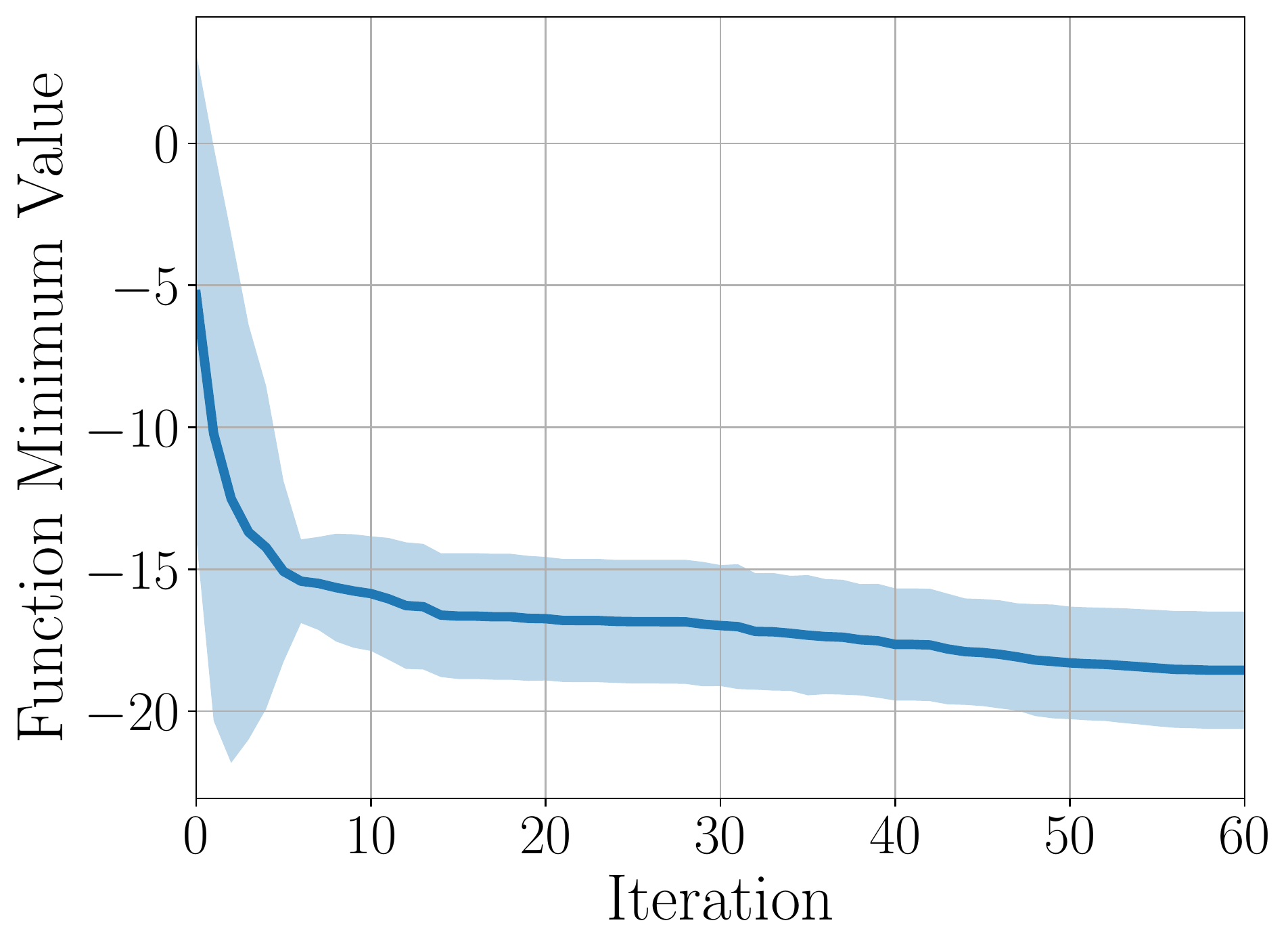}
			\label{fig:thm_203}
		}
		\setcounter{subfigure}{0}
		\subfigure[Cosines (8 dim.)]
		{
			\includegraphics[width=0.31\textwidth]{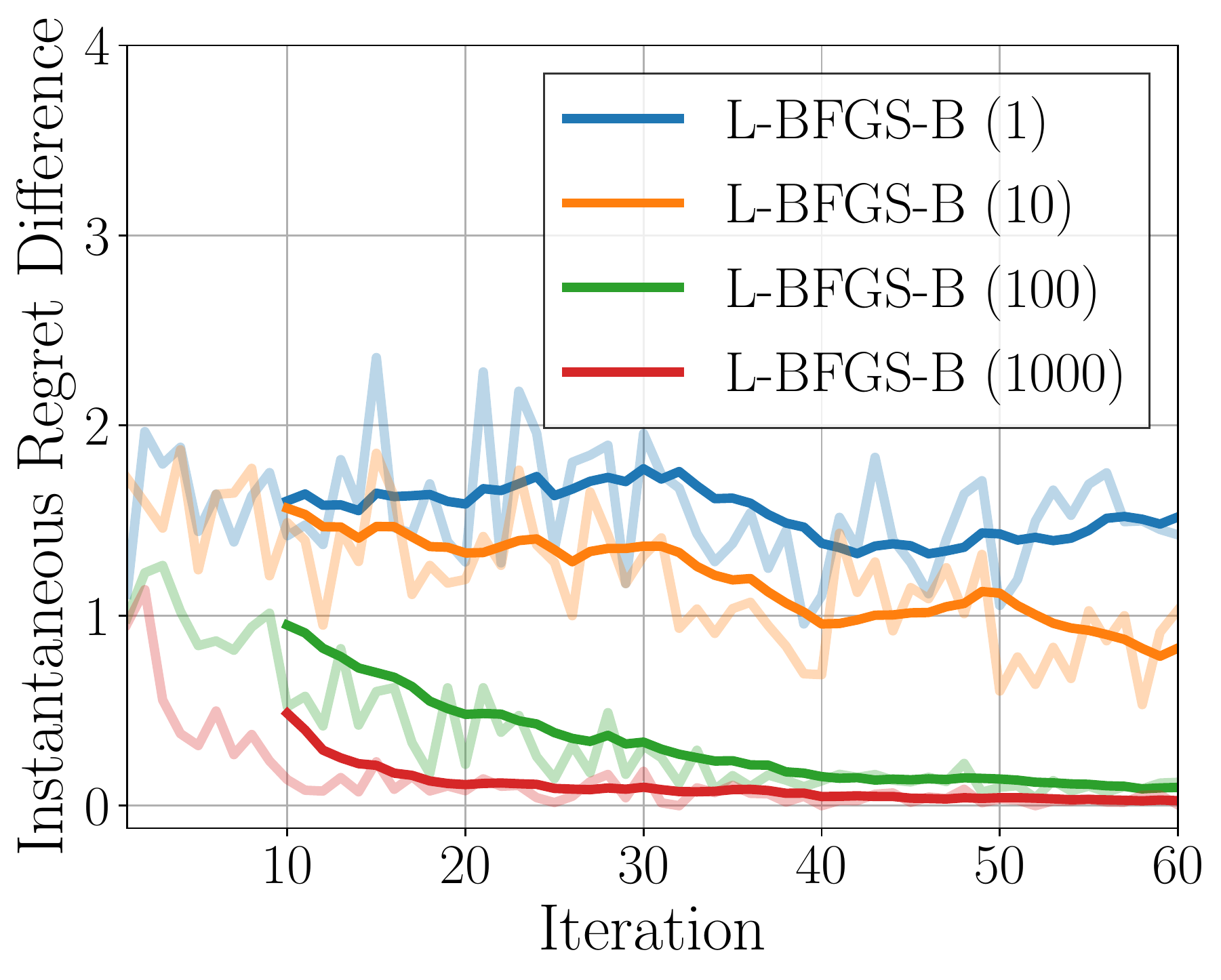}
			\label{fig:thm_211}
		}
		\subfigure[Hartmann6D]
		{
			\includegraphics[width=0.31\textwidth]{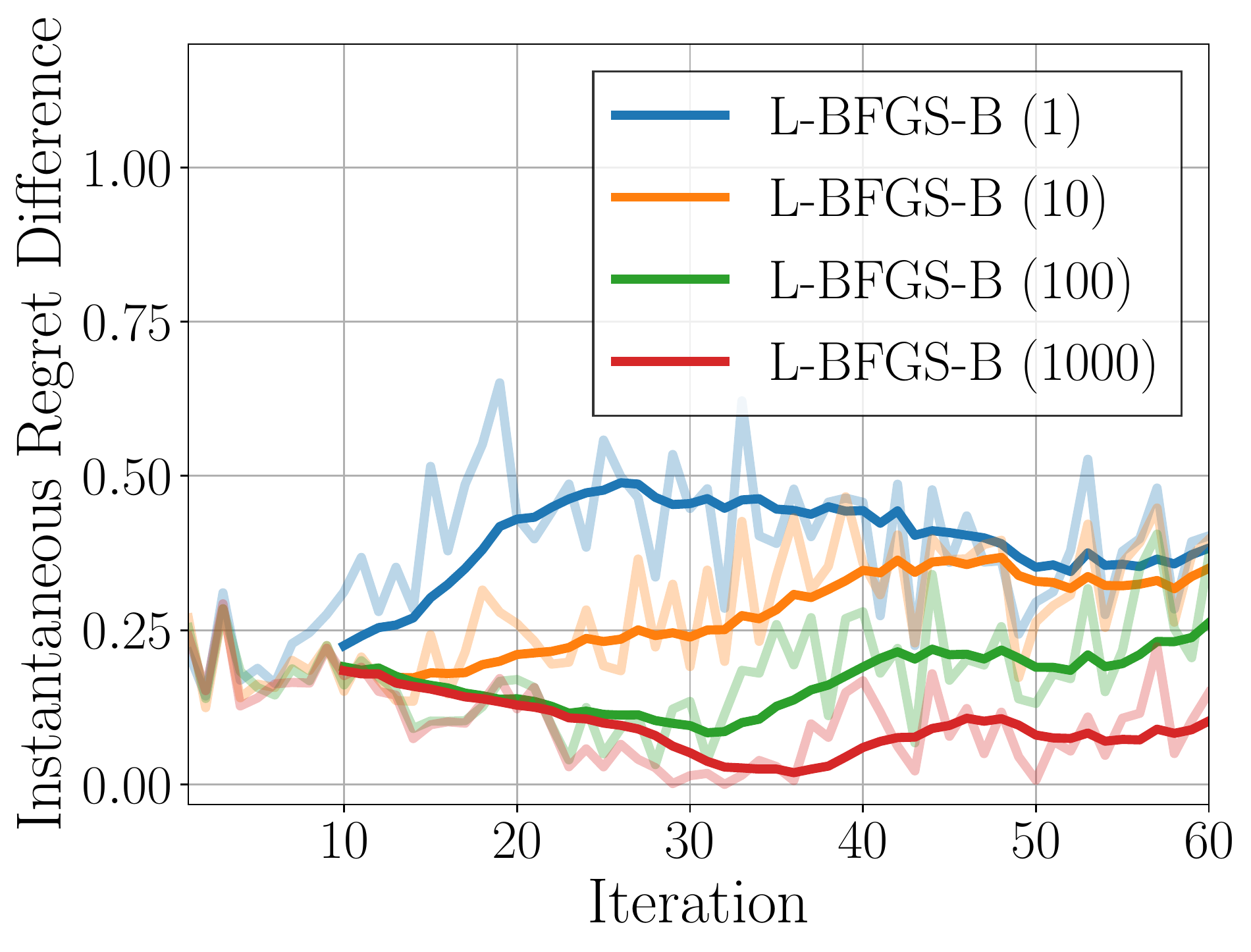}
			\label{fig:thm_212}
		}
		\subfigure[Holdertable]
		{
			\includegraphics[width=0.31\textwidth]{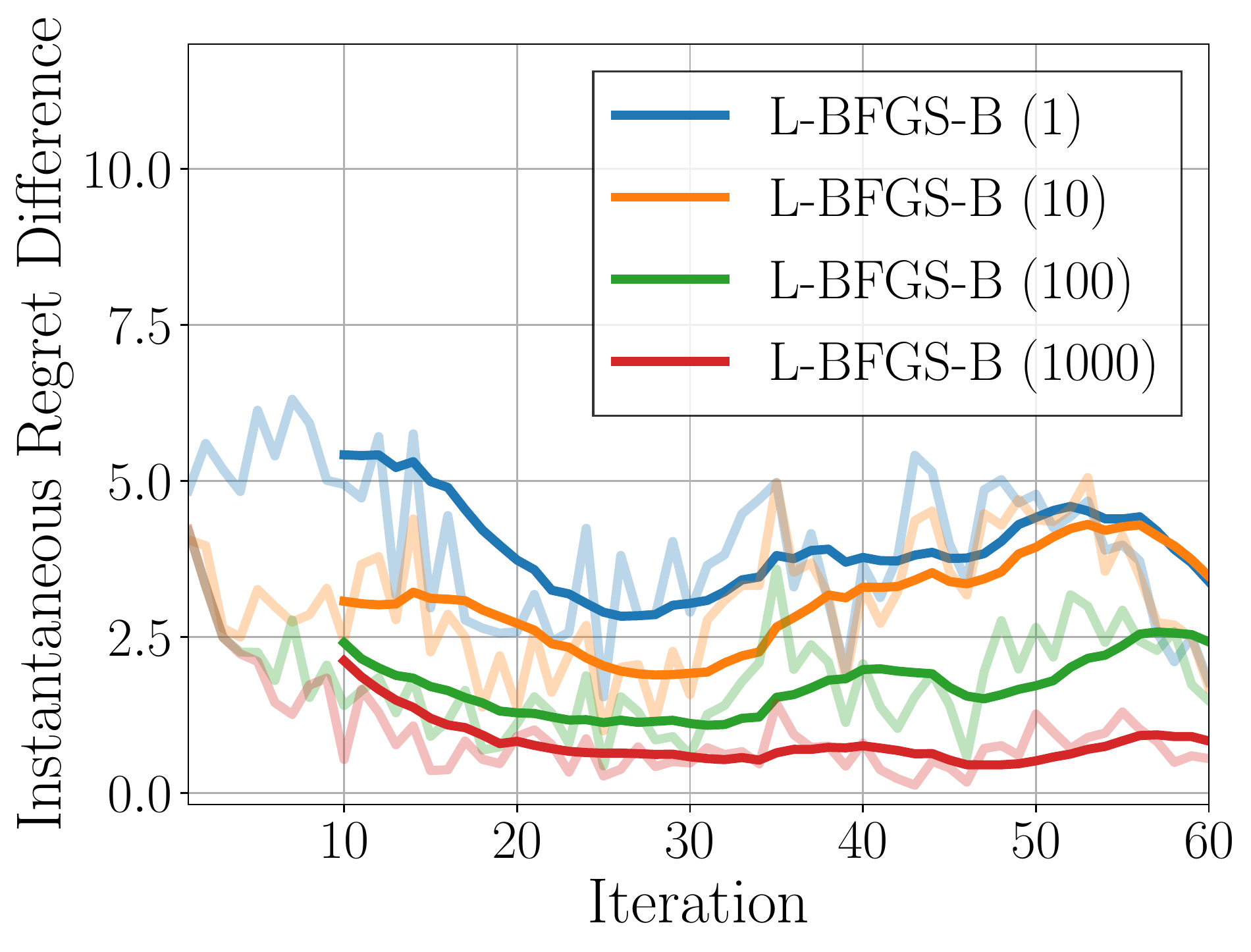}
			\label{fig:thm_213}
		}
	\caption{Empirical results on \thmref{thm:first} and \thmref{thm:second}.
	All settings follow the settings described in \figref{fig:val_thm_1}.}
	\label{fig:val_thm_2}
	\end{center}
\end{figure}

\begin{figure}[t]
	\begin{center}
		\subfigure
		{
			\includegraphics[width=0.31\textwidth]{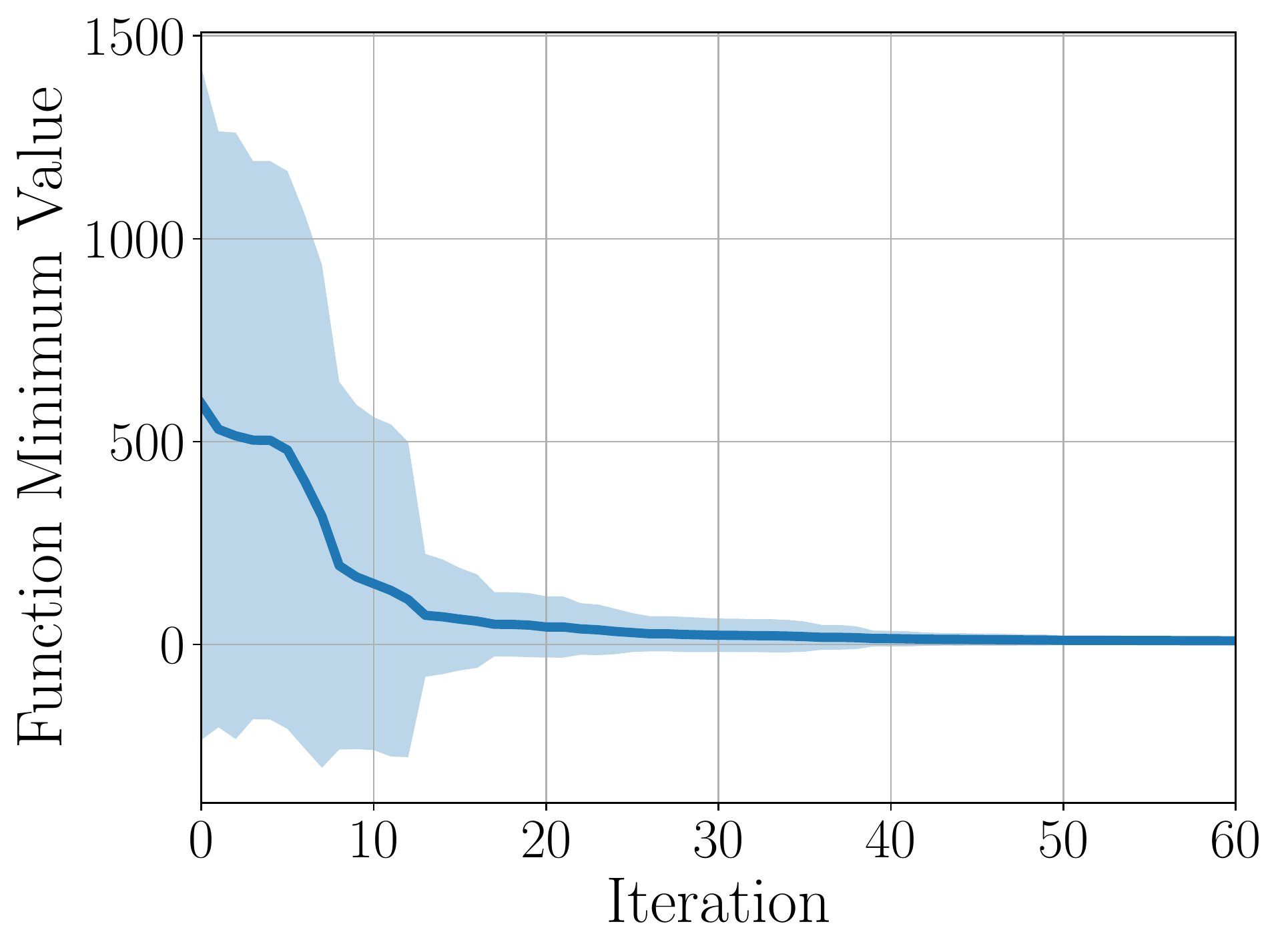}
			\label{fig:thm_301}
		}
		\subfigure
		{
			\includegraphics[width=0.31\textwidth]{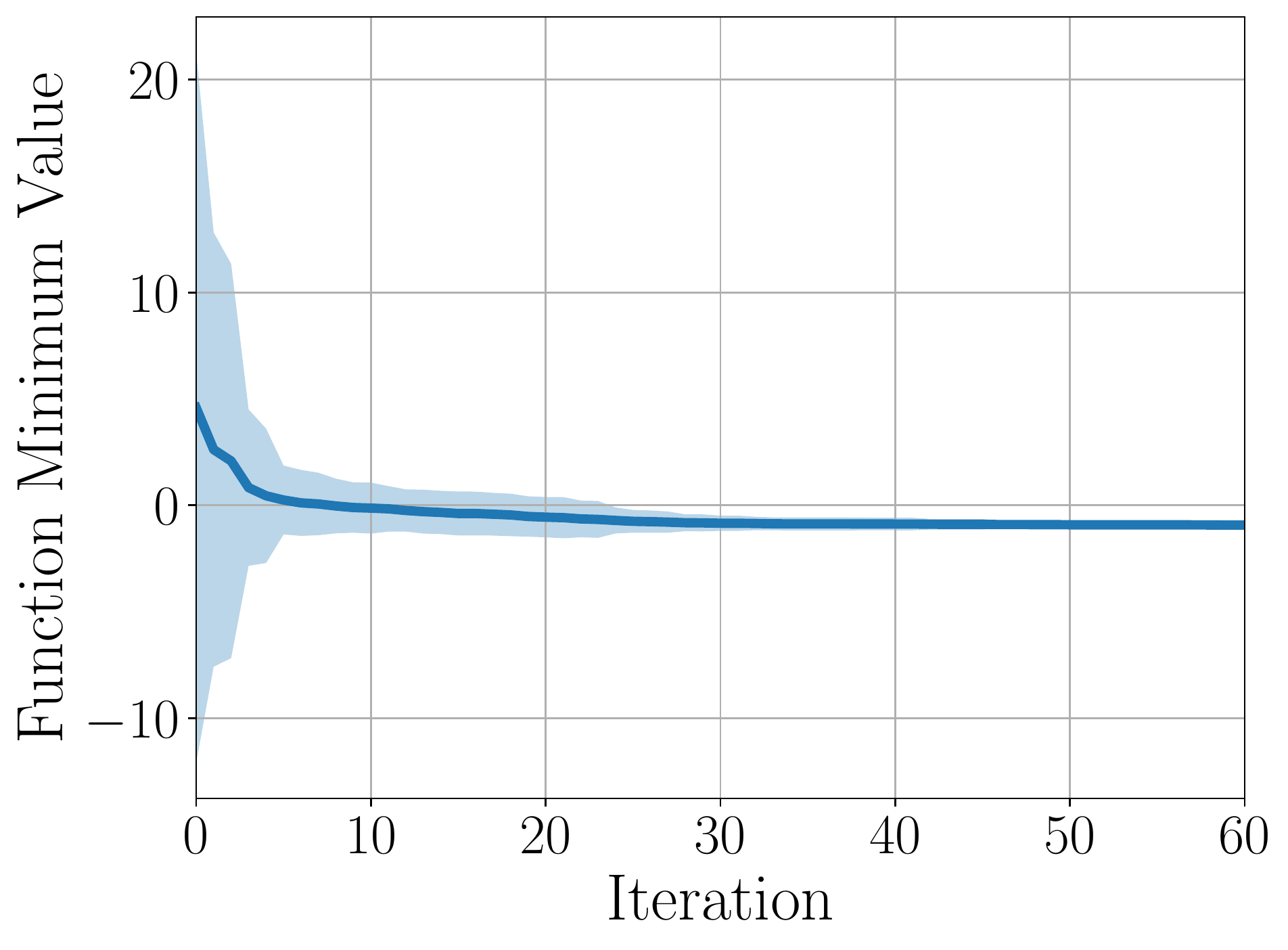}
			\label{fig:thm_302}
		}
		\subfigure
		{
			\includegraphics[width=0.31\textwidth]{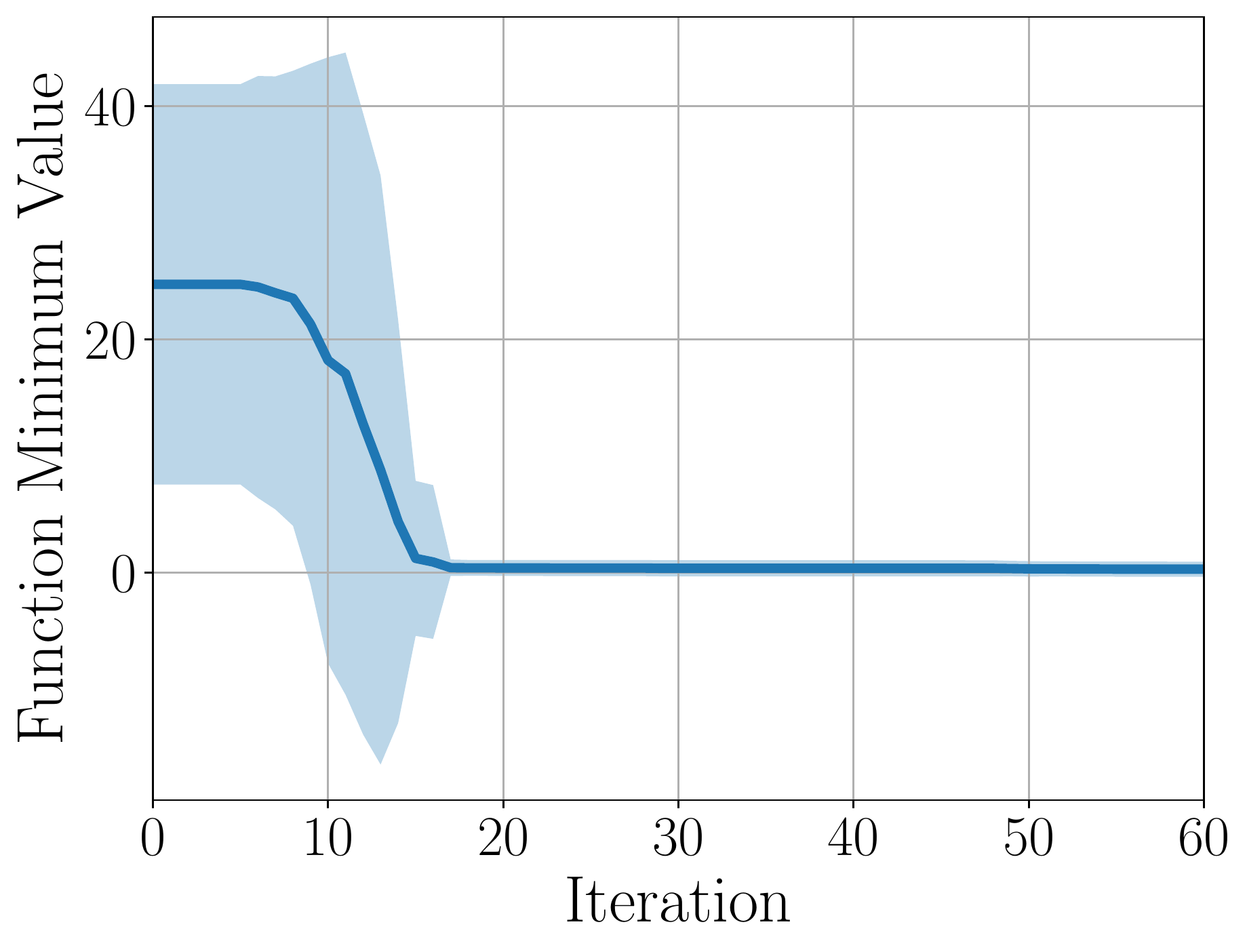}
			\label{fig:thm_303}
		}
		\setcounter{subfigure}{0}
		\subfigure[Rosenbrock]
		{
			\includegraphics[width=0.31\textwidth]{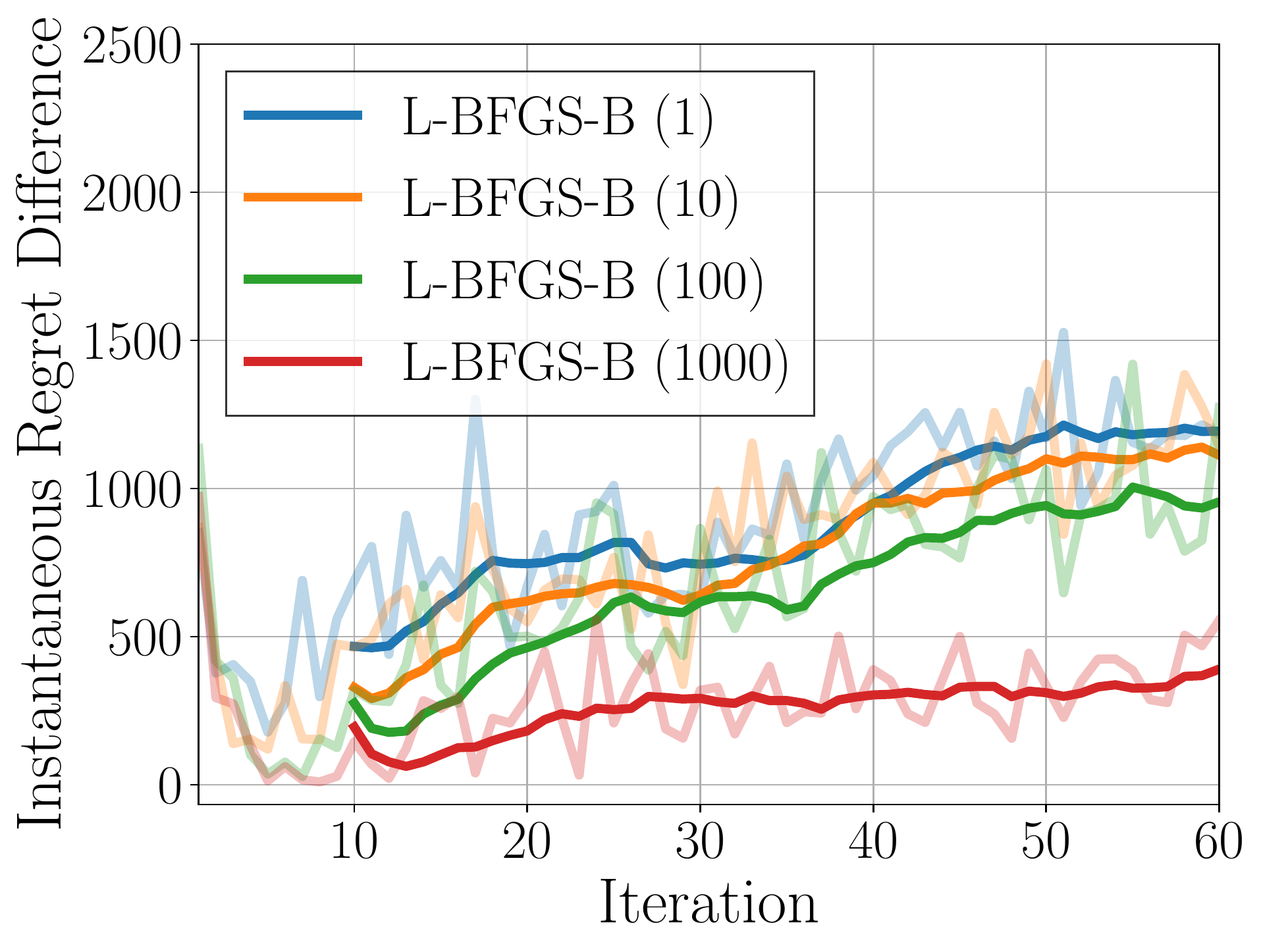}
			\label{fig:thm_311}
		}
		\subfigure[Six Hump Camel]
		{
			\includegraphics[width=0.31\textwidth]{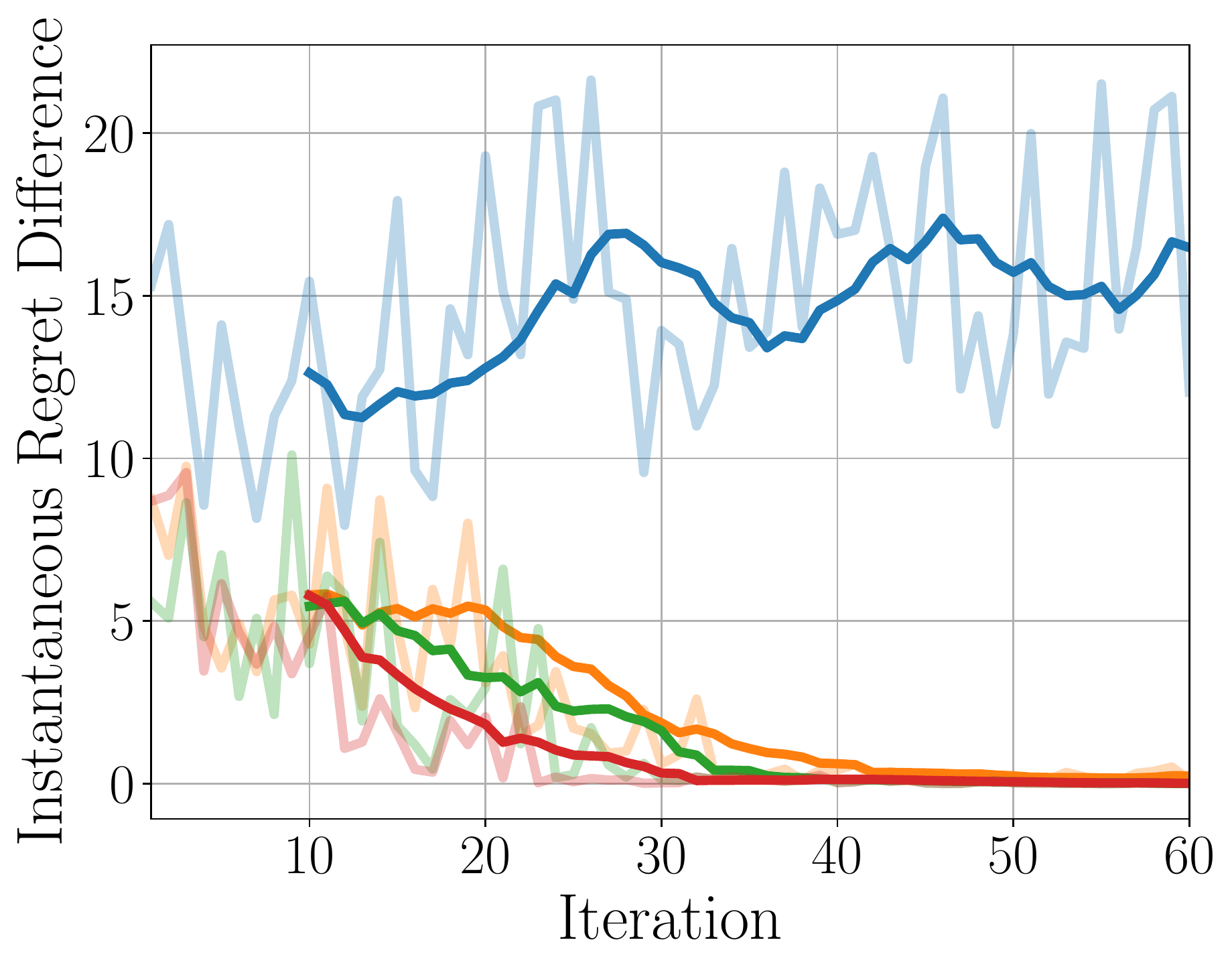}
			\label{fig:thm_312}
		}
		\subfigure[Sphere]
		{
			\includegraphics[width=0.31\textwidth]{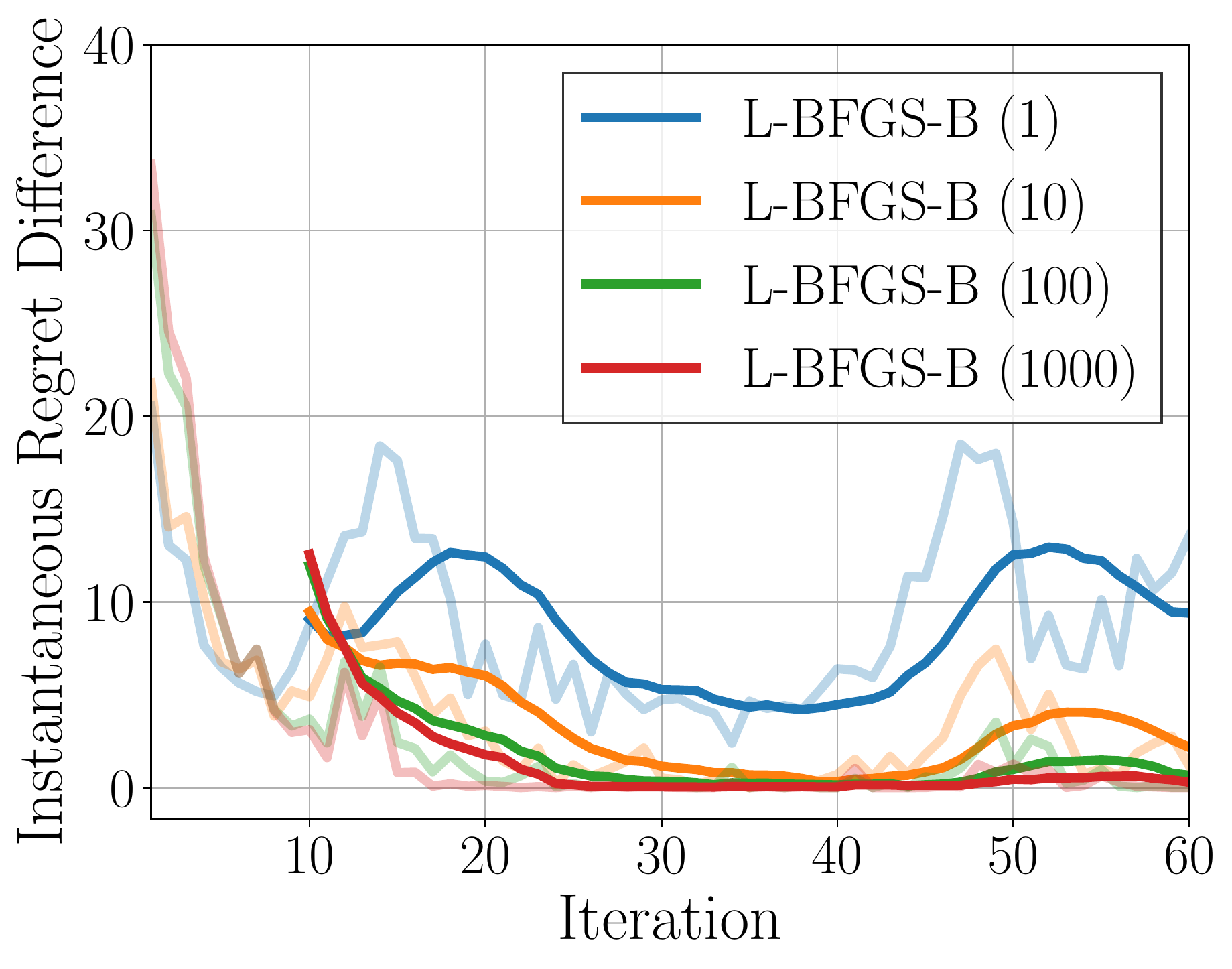}
			\label{fig:thm_313}
		}
	\caption{Empirical results on \thmref{thm:first} and \thmref{thm:second}.
	All settings follow the settings described in \figref{fig:val_thm_1}.}
	\label{fig:val_thm_3}
	\end{center}
\end{figure}

\subsection{On \thmref{thm:second}}
We extend \thmref{thm:first} into the version for a multi-started local optimizer defined in \defref{def:mslo_af}.
To prove the next theorem, we need to prove \lemref{lem:multi_local}.

% Lemma 9
\begin{lem}
	\label{lem:multi_local}
	Let the number of initial points for a multi-started local optimizer be $N$.
	A global optimizer and a multi-started local optimizer are different with a probability:
	\begin{equation}
		\bbP \big( \bx_{t, g} \neq \bx_{t, m} \big) = \left( 1 - \beta_g \right)^N,
		\label{eqn:lem_3_1}
	\end{equation}
	where $\bx_{t, m}$ is determined by \eqref{eqn:multi_local_acq}.
\end{lem}

\begin{proof}
	Since each initial condition of local optimizer is independently sampled, 
	$N$ local optimization methods started from different initial points are independently run.
	Therefore, the proof is obvious.
	\qed
\end{proof}

Because $N \geq 1$ and \eqref{eqn:lem_3_1} is less than one, 
\eqref{eqn:lem_3_1} is decreased as $N$ is increased.
For instance, these are satisfied:
\begin{equation}
	\left( 1- \beta_g \right)^N \leq 1 - \beta_g \quad \textrm{and} \quad
	\lim_{N \rightarrow \infty} \left( 1 - \beta_g \right)^N = 0. \label{eqn:smaller_1}
\end{equation}

By \lemref{lem:multi_local}, we can prove the theorem for the local optimization method started from multiple initial points.
Before introducing \thmref{thm:second}, 
we simply derive \corref{cor:distance}.

% Corollary 1
\begin{cor}
	\label{cor:distance}
	$l_2$ distance between the acquired points $\bx_{t, g}$ and $\bx_{t, m}$ from \eqref{eqn:global_acq} and \eqref{eqn:multi_local_acq} at iteration ${t}$ is larger than any $\gamma > \epsilon_3 > 0$ with a probability:
	\begin{equation}
		\bbP \big(\| \bx_{t, g} - \bx_{t, m} \|_2 \geq \epsilon_3 \big) \leq \frac{\gamma}{\epsilon_3} \left( 1 - \beta_g \right)^N.
	\end{equation}
\end{cor}

\begin{proof}
	Because it can be proved in the same manner of \lemref{lem:bounded_distance},
	it is trivial.
	\qed
\end{proof}

We provide the proof of \thmref{thm:second} with the above lemmas.

\begin{proof}
	It is an extension of \thmref{thm:first}.
	By \lemref{lem:lipschitz} and \corref{cor:distance}, it is proved in the same way.
	\qed
\end{proof}

As we mentioned before, because the equations in \eqref{eqn:smaller_1} are satisfied, 
we can emphasize a lower-bound on the probability of the case using a multi-started local optimizer is tighter than the case using a local optimizer.
It implies an appropriate multi-started local optimizer can produce a similar convergence quality 
with the global optimizer without expensive computational complexity.

\section{Empirical Analysis\label{sec:exp}}

We present empirical analyses for \thmref{thm:first} and \thmref{thm:second}, 
demonstrating the acquisition function optimization with global, local, and multi-started local optimizers on various examples:
Beale, Branin, Cosines (2 dim. and 8 dim.), Hartmann6D, Holdertable, 
Rosenbrock, Six Hump Camel, and Sphere functions, which are widely used as benchmark functions in the Bayesian optimization literature.
We use Gaussian process regression with M\'{a}tern 5/2 kernel as a surrogate function and EI as an acquisition function.
In addition, the hyperparameters (e.g., signal scale and lengthscales) 
of Gaussian process regression are found by maximizing the marginal likelihood.
All the experiments are implemented with \texttt{bayeso}~\cite{KimJ2017bayeso}.

To show regret differences, we need to sync up the historical observations for all the methods.
The Bayesian optimization results via DIRECT are compared to the results via local optimization methods,
considering each of the results by DIRECT as a true global solution for the acquisition function given.
At each iteration, four L-BFGS-B algorithms started from \{1, 10, 100, 1000\} different initial points 
find the next point to measure a regret difference.
The points found by L-BFGS-B are only used to compute the regret differences.
The transparent lines described in the lower panels of \figref{fig:val_thm_1} and \figref{fig:val_thm_2}
are the observed regret differences,
and the solid lines are the moving averages of the transparent lines,
each of which is computed as the unweighted mean of the previous 10 steps.

\begin{table}[t]
	\caption{Time (sec.) consumed in optimizing acquisition functions.}
	\label{tab:time}
	\begin{center}
	\begin{tabular}{cccccccccc}
		\toprule
		& \ref{fig:thm_111} & \ref{fig:thm_112} & \ref{fig:thm_113} & \ref{fig:thm_211} & \ref{fig:thm_212} & \ref{fig:thm_213} & \ref{fig:thm_311} & \ref{fig:thm_312} & \ref{fig:thm_313} \\
		\midrule
		DIRECT & 3.434 & 2.987 & 2.306 & 2.508 & 0.728 & 2.935 & 13.928 & 4.639 & 10.707 \\
		L-BFGS-B (1) & 0.010 & 0.004 & 0.052 & 0.023 & 0.026 & 0.017 & 0.005 & 0.010 & 0.030 \\
		L-BFGS-B (10) & 0.096 & 0.036 & 0.515 & 0.224 & 0.253 & 0.177 & 0.050 & 0.100 & 0.311 \\
		L-BFGS-B (100) & 0.977 & 0.363 & 5.173 & 2.224 & 2.533 & 1.760 & 0.504 & 0.969 & 3.048 \\
		L-BFGS-B (1000) & 9.720 & 3.633 & 51.818 & 22.306 & 25.305 & 17.629 & 5.049 & 9.682 & 30.764 \\
		\bottomrule
	\end{tabular}
	\end{center}
\end{table}

As shown in \figref{fig:val_thm_1} and \figref{fig:val_thm_2},
the regret difference at each iteration is decreased as $N$ is increased,
which supports the main theorems.
For some cases, the regret differences are slightly increased as the optimization step is repeated.
It means that the acquisition function at the latter iteration has relatively many local optima, 
which is usually observed in the Bayesian optimization procedures.
Furthermore, \tabref{tab:time} shows Bayesian optimization with a multi-started optimizer is a fair and efficient choice for most of the cases.
However, the cases using 1000-started optimizer tends to be slower than the ones using DIRECT,
which implies choosing the adequate number of initial conditions for a multi-started local optimizer is significant
and it should be carefully selected.

\section{Conclusion\label{sec:conclusion}}

In this paper, we theoretically and empirically analyze the upper-bound of instantaneous regret difference between two regrets occurred by global and local optimizers for an acquisition function.
The probability on this bound becomes tighter, using a multi-started local optimizer instead of the local optimizer.
Our experiments show our theoretical analyses can be supported.

%\bibliographystyle{splncs04}
%\bibliography{sjc}

%% BOB

%% EOB

\end{document}